\documentclass[twoside,11pt]{article}

\usepackage{arxiv}
\usepackage{microtype}
\usepackage{graphicx}
\usepackage{subcaption}
\usepackage{booktabs} % for professional tables

\usepackage{hyperref}

\hypersetup{
    colorlinks=true,
    linkcolor=black,
    citecolor=black,
    filecolor=black,
    urlcolor=black,
}

\usepackage{Definitions}
\usepackage{enumitem}
\usepackage{multirow} 
\usepackage[normalem]{ulem}
\usepackage{pinlabel}
\usepackage{times}

\usepackage{color}
\usepackage{amsmath}
\usepackage{amsfonts}

\jmlrheading{xx}{xx}{xx}{xx}{xx}{xx}{xx}

\usepackage{lipsum}

\ShortHeadings{Generalization in Deep Learning}{Kawaguchi, Kaelbling, and Bengio}
\firstpageno{1}

\begin{document}

\title{\vspace{-30pt} Generalization in Deep Learning \vspace{17pt}}

\author{Kenji Kawaguchi \\ MIT \and Leslie Pack Kaelbling \\ MIT 
        \and  Yoshua Bengio \\ University of Montreal \vspace{5pt} \normalsize}

\date{}

\maketitle

\begin{abstract}%   <- trailing '%' for backward compatibility of .sty file
This paper provides theoretical insights into why and how deep learning can generalize well, despite its large capacity, complexity, possible algorithmic instability, nonrobustness, and sharp minima, responding to an open question in the literature. We also discuss approaches to provide non-vacuous  generalization guarantees for deep learning. Based on theoretical  observations, we  propose new open problems and  discuss the limitations of our results. 
\end{abstract}

\let\thefootnote\relax\footnotetext{BibTeX of this paper is available at: \href{https://people.csail.mit.edu/kawaguch/bibtex.html}{\textcolor{blue}{https://people.csail.mit.edu/kawaguch/bibtex.html}} }

\section{Introduction}
Deep learning has seen significant  practical success and has had a profound impact on the conceptual bases of machine learning and artificial intelligence. Along with its practical success, the theoretical properties of deep learning have been a subject of active investigation. For \textit{expressivity} of neural networks, there are classical results regarding their universality \citep{leshno1993multilayer} and   exponential advantages over hand-crafted features  \citep{barron1993universal}.  Another series of theoretical studies  have considered how \textit{trainable} (or optimizable) deep hypothesis spaces are, revealing structural properties that may enable  non-convex optimization \citep{choromanska2015loss,kawaguchi2016deep}. However, merely having an \textit{expressive} and \textit{trainable} hypothesis space does not guarantee  good performance in predicting the values of future inputs, because of possible over-fitting to training data. This  leads to the study of \textit{generalization}, which is the focus of this paper.   

Some classical theory work attributes generalization ability to the use of a low-capacity class of hypotheses  \citep{vapnik1998statistical,mohri2012foundations}.
From the viewpoint of compact representation, which is related to  small capacity, it has been shown that deep hypothesis spaces have an exponential advantage over shallow hypothesis spaces for representing some classes of natural target functions  \citep{pascanu2013number,montufar2014number,livni2014computational,telgarsky2016benefits, poggio2017review}. In other words, when some assumptions implicit in the hypothesis space (e.g., deep composition of piecewise linear transformations) are approximately satisfied by the target function, one can achieve  very good generalization, compared to methods that do not rely on that assumption. However, a recent paper \citep{zhang2016understanding} has empirically shown that successful deep hypothesis spaces have sufficient capacity  to memorize random labels. This observation has been called an  ``apparent paradox''  and has led to active discussion by many    researchers \citep{arpit2017closer,krueger2017deep,hoffer2017train,wu2017towards,dziugaite2017computing,dinh2017sharp}.  \citet{zhang2016understanding} concluded with an open problem stating that  understanding such observations require rethinking generalization, while \citet{dinh2017sharp} stated that explaining why deep learning models can generalize well, despite  their overwhelming capacity, is an open area of research.

We begin, in Section \ref{sec:rethinking_ML}, by illustrating that, even in the case of linear models,  hypothesis spaces with overwhelming capacity  can result in arbitrarily small test errors and expected risks. Here, \textit{test error} is  the error of a learned hypothesis on data that it was not trained on, but which is often drawn from the same distribution.  Test error is a measure of how well the hypothesis generalizes to new data.
We closely examine this phenomenon, extending the original open problem from previous papers \citep{zhang2016understanding,dinh2017sharp} into a new open problem that strictly includes the original. We reconcile the possible apparent paradox by checking theoretical consistency and  identifying a difference in the underlying assumptions. Considering a difference in the focuses of theory and practice, we outline possible practical roles that generalization theory can play.   

Towards addressing these issues, Section
\ref{sec:generalization-training-validation} presents  generalization bou\-nds based on validation datasets, which can provide non-vacuous and numeri\-cally-tight generalization guarantees for deep learning in general. 
Section \ref{sec:theoretical_concern} analyzes generalization errors based on training datasets, focusing on a specific case of feed-forward neural
networks with ReLU units and max-pooling. Under these conditions, the developed   theory provides quantitatively tight       theoretical insights into the generalization behavior of neural networks.

\section{Background} \label{sec:preliminaries}
Let $x \in \mathcal{X}$ be an input and $y \in \mathcal{Y}$ be a target. Let ${\mathcal L}$ be a loss function. Let ${\mathcal R}[f]$ be the expected risk  of a function $f$, ${\mathcal R}[f]=\EE_{x,y\sim \PP_{(X,Y)}}[{\mathcal L}(f(x),y)]$, where $ \PP_{(X,Y)}$ is the true distribution. Let $f_{\mathcal A(S)}:\mathcal{X}  \rightarrow  \mathcal{Y}$ be a model learned by a learning algorithm $\mathcal A$ (including random seeds for simplicity) using a training dataset $S :=S_{m}:=\{(x_{1},y_{1}),\dots,(x_{m},y_{m})\}$ of size $m$. Let  ${\mathcal R}_{S}[f]$ be the empirical risk of $f$ as  ${\mathcal R}_{S}[f]=\frac{1}{m}\sum_{i=1}^m {\mathcal L}(f(x_i),y_i)$ with $\{(x_i,y_i)\}_{i=1}^m=S.$  Let $\mathcal F$ be a set of functions endowed with some structure or equivalently a \textit{hypothesis space}. Let $\mathcal L_{\mathcal F}$ be a family of loss functions associated with ${\mathcal F}$, defined by $\mathcal L_{\mathcal F}=\{g:f\in {\mathcal F}, g(x,y)= {\mathcal L}(f(x),y)\}$.
All vectors are \textit{column} vectors in this paper. For any given variable $v$, let $d_v$ be the dimensionality of the variable $v$.

A goal in machine learning is typically framed as the minimization of the expected risk ${\mathcal R}[f_{\mathcal A(S)}]$. We typically aim to minimize the non-computable expected risk ${\mathcal R}[f_{\mathcal A(S)}]$ by minimizing the computable  empirical risk ${\mathcal R}_{S}[\allowbreak f_{\mathcal A(S)}]$ (i.e., empirical risk minimization). One goal of  generalization theory is to explain and justify when and how minimizing ${\mathcal R}_{S}[f_{\mathcal A(S)}]$ is a sensible approach to minimizing ${\mathcal R}[f_{\mathcal A(S)}]$ by analyzing 
$$
\text{the generalization gap} := {\mathcal R}[f_{\mathcal A(S)}] - {\mathcal R}_{S}[f_{\mathcal A(S)}]. 
$$ 
 In this section only, we use the typical assumption  that $S$ is  generated by independent and identically distributed (i.i.d.) draws according to the true distribution $\PP_{(X,Y)}$; the following sections of this paper do not utilize this assumption. Under this assumption, a primary challenge of analyzing the generalization gap stems  from the \textit{dependence} of $f_{\mathcal A(S)}$ on the same dataset $S$ used in the definition of ${\mathcal R}_{S}$. 
Several  approaches in  \textit{statistical learning theory} have been developed to handle this dependence.

The \textit{hypothesis-space complexity} approach handles this dependence by decoupling $f_{\mathcal A (S)}$ from the particular $S$ by considering the worst-case gap for functions in the hypothesis space as 
$$
{\mathcal R}[f_{\mathcal A(S)}] - {\mathcal R}_{S}[f_{\mathcal A(S)}] \le \sup_{f\in {\mathcal F}} {\mathcal R}[f] - {\mathcal R}_{S}[f], 
$$    
and by carefully analyzing the right-hand side. Because the cardinality of ${\mathcal F}$ is typically (uncountably) infinite, a direct use of the union bound over all elements in ${\mathcal F}$ yields  a vacuous bound, leading to the need to  consider different quantities to characterize ${\mathcal F}$; e.g.,  Rademacher complexity and  the Vapnik--Chervonenkis  (VC) dimension. For example, if the codomain  of ${\mathcal L}$ is in $[0,1]$, we have \citep[Theorem 3.1]{mohri2012foundations} that for any $\delta >0$, with probability at least $1-\delta$,    
$$
\sup_{f\in {\mathcal F}} {\mathcal R}[f] - {\mathcal R}_{S}[f]  \le 2 \mathfrak{R}_m (\mathcal L_{{\mathcal F}}) + \sqrt\frac{\ln \frac{1}{\delta}}{2m},  
$$
where $ \mathfrak{R}_m (\mathcal L_{{\mathcal F}})$ is the Rademacher complexity of $\mathcal L_{{\mathcal F}}$, which then can be bounded by the Rademacher complexity of ${\mathcal F}$,  $\mathfrak{R}_m ({\mathcal F})$. For the  deep-learning hypothesis spaces ${\mathcal F}$, there are several  well-known bounds on $\mathfrak{R}_m ({\mathcal F})$ including those with explicit exponential dependence on depth \citep{sun2016depth,neyshabur2015norm,xie2015generalization} and   explicit linear dependence on the number of trainable parameters \citep{shalev2014understanding}. There has been significant  work on improving the bounds in this approach, but all existing solutions with this approach still depend on the complexity of a hypothesis space or a sequence of hypothesis spaces.

The \textit{stability} approach deals with the dependence of $f_{\mathcal A(S)}$ on the dataset $S$   by considering the \textit{stability} of algorithm $\mathcal A$ with respect to different datasets. The considered stability is a measure of how much changing a data point in $S$ can change $f_{\mathcal A(S)}$.  For example, if the algorithm $\mathcal A$ has uniform stability $\beta$ (w.r.t. ${\mathcal L}$) and if the codomain  of ${\mathcal L}$ is in $[0,M]$, we have \citep{bousquet2002stability} that for any $\delta >0$, with probability at least $1-\delta$, 
$$
{\mathcal R}[f_{\mathcal A(S)}] - {\mathcal R}_{S}[f_{\mathcal A(S)}] \le 2\beta +(4m\beta + M) \sqrt \frac{\ln \frac{1}{\delta}}{2m}.
$$   
Based on previous work on stability (e.g., \citealt{hardt2015train,kuzborskij2017data,gonen2017fast}), one may conjecture some reason for generalization in deep learning.     

The \textit{robustness} approach   avoids dealing with  certain details of the dependence of $f_{\mathcal A(S)}$ on $S$ by  
considering the robustness of algorithm $\mathcal A$ for all possible datasets. In contrast to stability, robustness is the measure of how much the loss value can vary w.r.t. \textit{the input space} of $(x,y)$.  For example, if algorithm $\mathcal A$ is $(\Omega,\zeta(\cdot))$-robust and the codomain of ${\mathcal L}$ is upper-bounded by $M$,  given a dataset $S$, we have \citep{xu2012robustness} that  for any $\delta>0$, with probability at least $1-\delta$, 
%$$ 
%\resizebox{0.99\hsize}{!}{$|{\mathcal R}[f_{\mathcal A(S)}] - {\mathcal R}_{S}[f_{\mathcal %A(S)}]| \le \zeta(S) + M \sqrt \frac{2 \Omega \ln 2 + 2 \ln \frac{1}{\delta}}{m}.$}
%$$
$$ 
|{\mathcal R}[f_{\mathcal A(S)}] - {\mathcal R}_{S}[f_{\mathcal A(S)}]| \le \zeta(S) + M \sqrt \frac{2 \Omega \ln 2 + 2 \ln \frac{1}{\delta}}{m}.
$$
The robustness approach requires an \textit{a priori known and fixed} partition of the input space such that the number of  sets in the partition is $\Omega$ and the change of loss values in each set of the partition is bounded by $\zeta(S)$ \textit{for all} $S$ (Definition 2 and the proof of Theorem 1 in \citealt{xu2012robustness}). In classification, if the \textit{margin}  is ensured to be large, we can fix the partition with balls of  radius corresponding to the large \textit{margin}, filling the input space.  Recently, this idea was applied to deep learning \citep{sokolic2017generalization,sokolic2017robust}, producing insightful   and effective generalization bounds,  while   still suffering from the curse of the dimensionality of the priori-known  fixed  input manifold.   
 
With regard to the above  approaches, \textit{flat minima} can be    viewed  as the concept of low variation in the \textit{parameter space}; i.e., a small perturbation in  the parameter space around a solution results in a small change in the loss surface.   Several studies  have provided arguments for generalization in deep learning based on  flat minima  \citep{keskar2017large}. However, \citet{dinh2017sharp}
showed that flat minima in practical deep learning hypothesis spaces can be turned into sharp minima via re-parameterization without affecting the generalization gap, indicating that it requires further investigation.

\section{Rethinking generalization  } \label{sec:rethinking_ML}

\citet{zhang2016understanding} empirically demonstrated that several deep hypothesis spaces can memorize random labels, while having the ability to produce zero training error and small test errors for particular natural datasets (e.g., CIFAR-10). They also empirically observed that  regularization on the norm of weights seemed to be unnecessary to obtain small test errors,  in contradiction to conventional  wisdom. These observations suggest the    following open problem:

\vspace{6pt}
\noindent \textbf{Open Problem 1.} Tightly characterize the expected risk ${\mathcal R}[f]$ or the generalization gap  ${\mathcal R}[f] - {\mathcal R}_{S}[f]$ with a sufficiently complex deep-learning hypothesis space ${\mathcal F} \ni f$, producing theoretical insights and  distinguishing the case of ``natural'' problem instances $(\PP_{(X,Y)},S)$ (e.g.,  images with natural labels) from the case of other problem instances  $(\PP_{(X,Y)}',S')$ (e.g., images with random labels).
\vspace{6pt}

Supporting and extending the empirical observations by \citet{zhang2016understanding}, we provide a theorem (Theorem \ref{thm:linear-couter}) stating  that the hypothesis space of over-parameterized linear models can memorize any training data \textit{and} decrease the training and test errors arbitrarily close to zero (including to zero) \textit{with} the norm of parameters being arbitrarily large, \textit{even when} the parameters are arbitrarily far from the ground-truth parameters. Furthermore,   Corollary \ref{coro:linear-counter} shows that conventional wisdom regarding the norm of the parameters $w$ can fail to explain  generalization, even in linear models that might seem not to be over-parameterized. 
All proofs for this paper are presented in the appendix.

\begin{theorem} 
\label{thm:linear-couter} 
Consider a linear model with the training prediction  $\hat Y(w)=\Phi w\in \RR^{m \times d_y}$, where $\Phi \in \RR^{m \times n}$ is a fixed feature matrix of the training inputs. Let $\hat Y_{\test}(w)=\Phi_\test w \in \RR^{m_{\test} \times d_y}$  be the test prediction, where $\Phi_{\test}\in \RR^{m_{\test} \times n}$ is a fixed  feature matrix of the test inputs.
Let $M=[\Phi^\top, \Phi_{\test}^\top]^\top$. Then, if $n > m$ and if $\rank(\Phi) = m$ and $\rank(M) < n$, 
\begin{enumerate}[label=(\roman*)] 
\item \label{thm:linear-counter-i} 
For any   $Y \in \RR^{m \times d_y}$, there exists a parameter $w'$ such that $\hat Y(w')= Y$, and  
\item \label{thm:linear-counter-ii}
if there exists a ground truth $w^*$ satisfying $Y=\Phi w^*$ and $Y_\test=\Phi_{\test} w^*$, then for any $\epsilon, \delta \ge 0$, there exists a parameter $w$ such that 

\begin{enumerate}[label=(\alph*)] 
\item  \label{thm:linear-counter-ii-a}
 $\hat Y(w)= Y+ \epsilon A$ for some matrix $A$ with $\|A\|_F\le1$, and
\item  \label{thm:linear-counter-ii-b}
$\hat Y_\test(w) = Y_\test + \epsilon B$ for some matrix $B$ with $\|B\|_F\le1$, and  
\item  \label{thm:linear-counter-ii-c} 
 $\|w\|_F\ge \delta$ and $\|w-w^*\|_F \ge \delta$.
\end{enumerate}
\end{enumerate} 
\end{theorem}
\begin{corollary} \label{coro:linear-counter}
If $n \le m$ and if $\rank(M) < n$, then statement \ref{thm:linear-counter-ii} in Theorem \ref{thm:linear-couter} holds.
\end{corollary}

Whereas Theorem \ref{thm:linear-couter} and  Corollary \ref{coro:linear-counter} concern test errors instead of  expected risk (in order to be consistent with  empirical studies), Remark \ref{prop:limit-ge}  shows the existence of   the same phenomena  for expected risk for general machine learning models  not limited  to deep learning  and linear hypothesis spaces; i.e., Remark \ref{prop:limit-ge}  shows that  none of small capacity, low complexity,  stability, robustness, and flat minima is  \emph{necessary}  for generalization in  machine learning for each given problem instance $(\PP_{(X,Y)},S)$, although one of them can be \textit{sufficient} for generalization.
This statement does not contradict  necessary conditions and no free lunch theorem from previous learning theory, as explained in the subsequent subsections.

\begin{remark}  \label{prop:limit-ge}
Given a pair $(\PP_{(X,Y)}, S)$ and a  desired $\epsilon > \inf_{f \in \Ycal^\Xcal} {\mathcal R}[f] - {\mathcal R}_{S}[f]$, let $f_\epsilon^*$ be a function such that $\epsilon \ge {\mathcal R}[f^*_{\epsilon}] - {\mathcal R}_{S}[f^*_{\epsilon}]$.
Then, 
\begin{enumerate}[label=(\roman*)]
\item \label{prop:ge-limit-1}
For any hypothesis space ${\mathcal F}$ whose hypothesis-space complexity is  large enough to memorize any dataset and which includes $f^*_\epsilon$ possibly at an arbitrarily sharp minimum, there exist learning algorithms $\mathcal A$ such that the generalization gap  of $f_{\mathcal A(S)}$ is at most $\epsilon$, and 
\item \label{prop:ge-limit-2} 
There exist    
arbitrarily unstable and arbitrarily non-robust algorithms $\mathcal A$ such that the generalization gap of $f_{\mathcal A(S)}$ is at most  $\epsilon$.
\end{enumerate}
To see this, first consider statement \ref{prop:ge-limit-1}. Given such a ${\mathcal F}$, consider any $\mathcal A$ such that $\mathcal A$ takes ${\mathcal F}$ and $S$ as input and outputs  $f^*_\epsilon$. Clearly, there are many such algorithms  $\mathcal A$. For example, given a $S$, fix $\mathcal A$ such that $\mathcal A$ takes ${\mathcal F}$ and $S$ as input and outputs $f^*_\epsilon$ (which already proves the statement), or even $f^*_\epsilon + \delta$ where $\delta$  becomes zero by the right choice of  hyper-parameters and of  small variations of ${\mathcal F}$ (e.g., architecture search in deep learning) such that ${\mathcal F}$ still satisfy the condition in the statement. This establishes  statement \ref{prop:ge-limit-1}. 

Consider statement \ref{prop:ge-limit-2}. Given any dataset $S'$, consider a look-up algorithm $\mathcal A'$ that always outputs $f^*_\epsilon$ if $S=S'$, and outputs $f_1$ otherwise such that $f_1$ is arbitrarily non-robust and $|{\mathcal L}(f^{*}_\epsilon(x),y)-{\mathcal L}(f_1(x),y)|$ is arbitrarily large (i.e., arbitrarily non-stable). This proves statement \ref{prop:ge-limit-2}. Note that while $A'$ here suffices to prove statement \ref{prop:ge-limit-2}, we can also generate other non-stable and non-robust algorithms  by noticing the essence captured in Remark \ref{rem:essence}.
\end{remark}

 We capture  the essence of all the above observations in the following remark. 

\begin{remark} \label{rem:essence}
The expected risk ${\mathcal R}[f]$ and the generalization gap  ${\mathcal R}[f] - {\mathcal R}_{S}[f]$  of a hypothesis $f$  with a true distribution $\PP_{(X,Y)}$ and a dataset $S$    are completely determined by the tuple $(\PP_{(X,Y)},S,f)$, independently of other factors, such as a hypothesis space ${\mathcal F}$ (and hence its properties such as capacity, Rademacher complexity, pre-defined bound on  norms, and flat-minima) and    properties of random  datasets different from the given $S$ (e.g., stability and robustness of the learning algorithm $\mathcal A$).  In contrast, the conventional wisdom states that   these other factors  are what matter. This has created the ``apparent paradox'' in the literature.   
\end{remark}
From these observations, we propose  the following  open problem:

\vspace{6pt}
\noindent\textbf{Open Problem 2.} 
Tightly characterize the expected risk ${\mathcal R}[f]$ or the generalization gap  ${\mathcal R}[f] - {\mathcal R}_{S}[f]$ of a hypothesis $f$ with a pair $(\PP_{(X,Y)},S)$ of a true distribution and a dataset, \textit{producing theoretical insights,} based only on  properties of  the hypothesis $f$ and the pair $(\PP_{(X,Y)},S)$.
\vspace{6pt}

Solving Open Problem 2 for deep learning implies solving Open Problem 1, but not vice versa. Open Problem 2  encapsulates  the essence of Open Problem 1 and all the issues from our Theorem \ref{thm:linear-couter}, Corollary \ref{coro:linear-counter} and Proposition \ref{prop:limit-ge}. 

\subsection{Consistency of theory } \label{app:consistency}
 The empirical observations  in \citep{zhang2016understanding} and our results above  may seem to contradict the  results of statistical learning theory. However, there is no contradiction, and the apparent inconsistency arises from the misunderstanding and misuse of the precise  meanings of  the theoretical statements. 

Statistical learning theory can be considered to provide two types of statements relevant to the scope of this paper. The first type (which comes from upper bounds) is logically in the form of  ``$p$  implies $q$,'' where   $p:=$ ``the hypothesis-space complexity is small'' (or another statement about stability, robustness,  or flat minima), and $q:=$ ``the generalization gap is  small.'' Notice that ``$p$ implies $q$'' does not  imply ``$q$ implies $p$.'' Thus, based on statements of this type, it is entirely possible that the generalization gap is small even when the hypothesis-space complexity is large or the learning mechanism is unstable, non-robust, or subject to sharp minima. 

The second type (which comes from lower bounds) is logically in the following form: in a set $U_{\all}$   of all possible problem configurations, there exists a   subset $U\subseteq U_\all$ such that  ``$q$ implies $p$'' in $U$ (with the same definitions of $p$ and $q$ as in the previous paragraph). For example, \citet[Section 3.4]{mohri2012foundations} derived  lower bounds on the generalization gap by showing the existence of a ``bad'' distribution that characterizes $U$. Similarly, the classical \textit{no free lunch} theorems are the results with the existence of a worst-case distribution for each algorithm. However, if the problem instance at hand (e.g., object classification with MNIST or CIFAR-10) is not in such a $U$  in the proofs (e.g., if the data distribution is not among the ``bad'' ones considered in the proofs),  $q$ does not necessarily imply $p$. Thus, it is still naturally possible  that the generalization gap is small with large  hypothesis-space complexity,   instability, non-robustness,  and sharp minima. Therefore, there is no contradiction or paradox.

\subsection{Difference in assumptions and problem settings} \label{sec:reslove_ge_puzzle}

Under certain assumptions, many results in statistical learning theory have been shown to be tight and insightful (e.g., \citealt{mukherjee2006learning,mohri2012foundations}).
Hence, the need of rethinking generalization partly comes from a difference in the assumptions and problem settings.

 Figure \ref{fig:scopes_slt-alt} illustrates the differences in  assumptions  in statistical learning theory and some empirical studies.  On one hand, in statistical learning theory, a distribution $\PP_{(X,Y)}$ and a dataset $S$ are usually unspecified   except    that $\PP_{(X,Y)}$ is in some set $\mathcal P$ and a dataset $S \in D$ is drawn randomly according to  $\PP_{(X,Y)}$ (typically with the i.i.d. assumption). On the other hand, in most empirical studies  and in our theoretical results (Theorem \ref{thm:linear-couter} and Proposition \ref{prop:limit-ge}), a distribution $\PP_{(X,Y)}$ is still unknown yet specified (e.g., via a real world  process) and a dataset $S$ is  specified  and usually known  (e.g.,  CIFAR-10 or ImageNet). Intuitively, whereas statistical learning theory considers a  set $\mathcal P \times D $ because of  weak assumptions, some empirical studies can focus on a specified point $(\PP_{(X,Y)},S)$ in a set $\mathcal P \times D$ because of stronger  assumptions. Therefore, by using the same terminology  such as  ``expected risk'' and ``generalization'' in both cases, we are susceptible to confusion and apparent contradiction.

\begin{figure}[tb!]
\centering
\includegraphics[width=0.6\columnwidth]{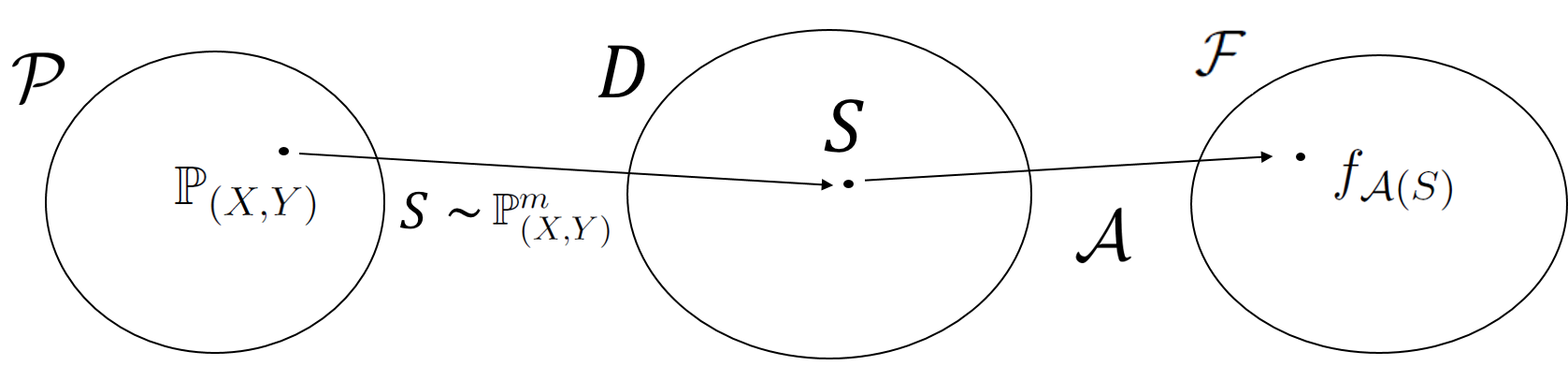} 
\caption{An illustration of differences in assumptions.  Statistical learning theory analyzes the generalization behaviors of $f_{\mathcal A(S)}$ over randomly-drawn \textit{unspecified} datasets $S \in D$ according to some  \textit{unspecified}  distribution $\PP_{(X,Y)}\in \mathcal P$. Intuitively,  statistical learning theory concerns more about questions regarding a set $\mathcal P \times D$ because of   the \textit{unspecified} nature of $(\PP_{(X,Y)},S)$, whereas certain empirical studies (e.g., \citealt{zhang2016understanding}) can focus  on questions regarding each  \textit{specified}   point $(\PP_{(X,Y)},S)\in \mathcal P \times D$.}
\label{fig:scopes_slt-alt} 
\end{figure}

Lower bounds, necessary conditions and tightness in statistical learning theory are typically defined via a worst-case distribution $\PP_{(X,Y)}^{\text{worst}}\in \mathcal P$. For instance, classical ``no free lunch'' theorems and certain lower bounds on the generalization gap (e.g., \citealt[Section 3.4]{mohri2012foundations}) have been proven \textit{for} the worst-case distribution $\PP_{(X,Y)}^{\text{worst}} \in \mathcal P$. Therefore, ``tight'' and ``necessary''  typically mean ``tight'' and ``necessary'' for the set $\mathcal P \times D $ (e.g., through the worst or average case), but \textit{not} for each particular point $(\PP_{(X,Y)},S) \in \mathcal P \times D $. From this viewpoint, we can understand that even if the quality of the set $\mathcal P \times D $ is ``bad''  overall, there may exist a ``good'' point $(\PP_{(X,Y)},S) \in \mathcal P \times D $.

Several approaches in statistical learning theory, such as the data dependent and Bayesian  approaches  \citep{herbrich2002algorithmic,dziugaite2017computing},  use more assumptions  on the set $\mathcal P \times D$ to take advantage of more prior and posterior information; these have an ability to tackle  Open Problem 1. However, these approaches do not apply to Open Problem 2 as  they still depend on other factors than the given $(\PP_{(X,Y)},S,f)$. For example, data-dependent bounds with the \textit{luckiness framework} \citep{shawe1998structural,herbrich2002algorithmic} and \textit{empirical} Rademacher complexity  \citep{koltchinskii2000rademacher,bartlett2002model} still depend on a  concept of hypothesis spaces (or the sequence of hypothesis spaces), and the robustness approach \citep{xu2012robustness} depend on different datasets than a given $S$ via the definition of robustness (i.e.,  in Section \ref{sec:preliminaries},  $\zeta(S)$ is a data-dependent term, but the definition of $\zeta$ itself and $\Omega$ depend on other datasets than $S$). 

We note that analyzing a set $\mathcal P \times D$  is of  significant interest  for its own merits and is a natural task in the field of computational complexity (e.g., categorizing a set of problem instances into subsets  with or without  polynomial solvability).  Indeed, the situation  where theory focuses more on a set and many practical studies focus on each element in the set  is prevalent in computer science (see the discussion in Appendix \ref{app:relation_other-fields} for more detail).

\subsection{Practical role of generalization theory} \label{app:role_ge}

From the discussions above, we can see  that there  is a \textit{logically expected} difference between the scope in theory and the focus in practice; it is logically expected that there are problem instances where theoretical bounds are pessimistic. In order for generalization theory to have maximal impact in practice, we must be clear on a set of different roles it can play regarding practice, and then work to extend and strengthen it in each of these roles.  We have identified the following practical roles for theory: 

\begin{description}
\item[Role 1] \label{rol_expected_risk} 
Provide guarantees on expected risk.
\item[Role 2] \label{rol_theory}
Guarantee generalization gap
\begin{description}
\item[Role 2.1] \label{rol_expected_risk_1} 
to be small for a given fixed  $S$, and/or
\item[Role 2.2] \label{rol_expected_risk_2}
to approach zero with \textit{a fixed model class} as $m$ increases.
\end{description}
\item[Role 3] \label{rol_guide_search}
Provide theoretical insights to guide the search over model classes.
\end{description}

\section{Generalization bounds via validation} \label{sec:generalization-training-validation}
In practical deep learning, we typically adopt the training--validation para\-digm, usually with a held-out validation set. We then  search over hypothesis spaces by changing architectures (and other hyper-parameters) to obtain low validation error. In this view, we can conjecture the reason that deep learning can sometimes generalize well as follows: it is partially because we can obtain a  good model via  search using a validation dataset.
Indeed, Proposition \ref{prop:ge_validation} states that if the validation error of a hypothesis is small,  it is guaranteed to generalize well,  regardless of its capacity, Rademacher complexity, stability, robustness, and flat minima. Let  $S_{m_\val}^{(\val)}$ be a held-out validation dataset of size $m_\val$, which is independent of the training dataset $S$.

\begin{proposition} \label{prop:ge_validation}
\emph{(generalization guarantee via validation error)}
Assume that  $S_{m_\val}^{(\val)}$ is  generated by i.i.d. draws according to a true distribution $\PP_{(X,Y)}$. Let $\kappa_{f,i} = {\mathcal R}[f] - {\mathcal L}(f(x_i),y_i)$ for $(x_i,y_i) \in S_{m_\val}^{(\val)}$.  Suppose that $\EE[\kappa_{f,i}^2] \le \gamma^{2}$ and $|\kappa_{f,i}| \le C$ almost surely, for all $(f,i) \in {\mathcal F}_\val \times \{1,\dots,m_\val\}$. Then, for any $\delta>0$, with probability at least $1-\delta$, the following holds for all $f \in {\mathcal F}_\val$: 
$$
{\mathcal R}[f] \le {\mathcal R}_{S_{m_\val}^{(\val)}}[f] + \frac{2C \ln (\frac{|{\mathcal F}_\val|}{\delta})}{3m_\val}+\sqrt{ \frac{2\gamma^2 \ln(\frac{|{\mathcal F}_\val|}{\delta})}{m_\val}}.
$$
\end{proposition}

Here,  ${\mathcal F}_\val$ is defined as a  set of models $f$ that is independent of a held-out validation dataset $S_{m_\val}^{(\val)}$, but can depend on the training dataset $S$.   For example, ${\mathcal F}_\val$ can contain a  set of models $f$ such that each element $f$ is a result  at the end of each epoch during training with at  least $99.5$\% \textit{training} accuracy. In this example, $|{\mathcal F}_\val|$ is at most (the number of epochs) $\times$ (the cardinality of the set of possible hyper-parameter settings), and is likely much smaller than that because of the 99.5\% training accuracy criteria and the fact that a  space of many hyper-parameters is narrowed down by using the training dataset as well as other datasets from different tasks. If a  hyper-parameter search depends on the validation dataset, ${\mathcal F}_\val$ must be the possible space of the search instead of the space actually visited by the search. We can also use a sequence  $\{{\mathcal F}_{\val}^{(j)}\}_j$ (see Appendix \ref{app:additional_discussions}).

The bound in Proposition \ref{prop:ge_validation} is non-vacuous and tight enough to be practically meaningful. For example, consider a classification task with 0--1 loss. Set $m_\val= 10,000$ (e.g., MNIST and CIFAR-10) and  $\delta=0.1$. Then, even in the  worst case with $C= 1$ and $\gamma^{2} = 1$ and even with $|{\mathcal F}_\val| = 1,00 0,000,000$, we have with probability at least $0.9$ that ${\mathcal R}[f] \le {\mathcal R}_{S_{m_\val}^{(\val)}}[f] + 6.94\%$ for all $f \in {\mathcal F}_\val$. In a non-worst-case scenario, for example, with  $C= 1$ and $\gamma^{2} = (0.05)^{2}$, we can replace $6.94\%$ by $0.49\%$. With a larger validation set (e.g.,  ImageNet) and/or more optimistic $C$ and $\gamma^{2} $, we can  obtain much better bounds.

Although Proposition \ref{prop:ge_validation}  poses the  concern of increasing the generalization bound when using a single validation dataset with too large $|{\mathcal F}_\val|$, the rate of increase is only $\ln |{\mathcal F}_\val|$ and $\sqrt{\ln |{\mathcal F}_\val|}$.  We can also avoid   dependence on the cardinality of ${\mathcal F}_{\val}$  using Remark \ref{rem:ge_validation_2}.

\begin{remark} \label{rem:ge_validation_2} 
Assume that  $S_{m_\val}^{(\val)}$ is  generated by i.i.d. draws according to  $\PP_{(X,Y)}$. Let $\mathcal L_{{\mathcal F}_{\val}}=\{g:f\in {\mathcal F}_{\val}, g(x,y):={\mathcal L}(f(x),y)\}$. By applying  \citep[Theorem 3.1]{mohri2012foundations} to $\mathcal L_{{\mathcal F}_{\val}}$, if the codomain  of ${\mathcal L}$ is in $[0,1]$, with probability at least $1-\delta$, for all $f \in \mathcal F_{\val}$,
$
{\mathcal R}[f] \le {\mathcal R}_{S_{m_\val}^{(\val)}}[f] +   2 {\mathfrak{R}}_m (\mathcal L_{{\mathcal F}_{\val}}) + \sqrt{(\ln 1/\delta)/m_\val}.
$
\end{remark}

Unlike the standard use of  Rademacher complexity with a training dataset, the set ${\mathcal F}_{\val}$ cannot depend on the validation set $S_{m_\val}$, but can depend  on the training dataset $S$ in any manner, and hence ${\mathcal F}_{\val}$  differs significantly from the typical hypothesis space defined by the parameterization of models. We can thus end up with a very different effective capacity and hypothesis complexity (as selected by model search using the validation set) depending on whether the training data are random or have interesting structure which the neural network can capture.

\section{Direct analyses of neural networks} \label{sec:theoretical_concern}
Unlike the previous section, this section analyzes the generalization gap with a training dataset $S$. In Section \ref{sec:rethinking_ML}, we  extended Open Problem 1 to Open Problem 2, and identified the different assumptions in theoretical and empirical studies. Accordingly, this   section aims to address these problems to some extent, both in the case of particular specified datasets and the case of random  unspecified datasets.
To  achieve this goal, this section presents a \textit{direct analysis} of neural networks, rather than deriving results about neural networks from more generic theories based on capacity, Rademacher complexity, stability, or robustness.

Sections \ref{sec:data_dependent_bounds} and \ref{sec:sub_data-independent} deals with the squared loss, while Section \ref{sec:classification_direct} considers   0-1 loss with multi-labels.

\subsection{Model description via deep paths}  \label{sec:direct_model_description}
We consider general  neural networks of any depth that have the structure of a directed acyclic graph (DAG) with ReLU nonlinearity and/or max pooling. This   includes  any structure of a feedforward network  with convolutional and/or fully connected layers, potentially with skip connections. For pedagogical purposes, we first discuss our model description for layered networks without skip connections, and then describe  it for DAGs.

\paragraph{Layered nets without skip connections} Let $z^{[l]}(x,w)\in \RR^{n_l}$ be the pre-activation vector of the $l$-th hidden layer, where $n_l$ is the width of the $l$-th hidden layer, and  $w$ represents the trainable parameters.  Let $L-1$ be the number of hidden layers.  For layered networks without skip connections,  the pre-activation (or pre-nonlinearity) vector of the $l$-th layer can be written as  
$$
z^{[l]}(x,w) = W^{[l]} \sigma^{(l-1)} \left(z^{[l-1]}(x,w)\right),  
$$
with a boundary definition $\sigma^{(0)} \left(z^{[0]}(x,w)\right)\equiv x$, where $\sigma^{(l-1)}$ represents  nonlinearity via ReLU and/or max pooling at the $(l-1)$-th hidden layer, and $W^{[l]}\in \RR^{n_l \times n_{l-1}}$ is a matrix of weight parameters connecting the $(l-1)$-th layer to the $l$-th layer. Here, $W^{[l]}$ can have \textit{any} structure (e.g., shared    and sparse weights to represent a convolutional layer).  Let  $\dot \sigma^{[l]}(x,w)$ be  a vector with each element being $0$ or $1$ such that $\sigma^{[l]} \left(z^{[l]}(x,w)\right)=\dot \sigma^{[l]}(x,w)\circ z^{[l]}(x,w)$, which is an    element-wise product of  the vectors $\dot \sigma^{[l]}(x,w)$ and $ z^{[l]}(x,w)$. Then, we can write the pre-activation of the $k$-th output unit at the last layer $l=L$ as
\begin{align*}
z^{[L]}_k(x,w)=\sum_{j_{L-1}=1}^{n_{L-1}}W^{[L]}_{kj_{L-1}} \dot \sigma_{j_{L-1}}^{(L-1)}(x,w) z_{j_{L-1}}^{[L-1]}(x,w).
\end{align*}
By expanding $ z^{[l]}(x,w)$ repeatedly and exchanging the sum and product via the distributive law of multiplication,
\begin{align*}
z^{[L]}_k(x,w) 
 = 
%\scalebox{.96}{$
\sum\limits_{j_{L-1}=1}^{n_{L-1}} \sum\limits_{j_{L-2}=1}^{n_{L-2}} \dots \sum\limits_{j_{0}=1}^{n_{0}}\overline W_{kj_{L-1}j_{L-2}\dots j_{0}} \dot \sigma_{j_{L-1}j_{L-2}\dots j_{1}}(x,w)x_{j_{0}},
%$}
\end{align*}
\normalsize
where $\overline W^{}_{kj_{L-1}j_{L-2}\dots j_{0}}=W^{[L]}_{kj_{L-1}} \prod_{l=1}^{L-1} W^{[l]}_{j_{l}j_{l-1}}$ and $\dot \sigma_{j_{L-1}j_{L-2}\dots j_{1}}(x,w)= \allowbreak \prod_{l=1}^{L-1}  \allowbreak \dot \sigma_{j_{l}}^{[l]}(x,w)$. By merging the indices $j_{0},\dots,j_{L-1}$ into $j $ with some bijection between $\{1,\dots,n_{0}\} \times \cdots \times \{1,\dots, \allowbreak n_{L-1}\} \ni (j_{0},\dots,j_{L-1})$ and $\{1,\dots,n_{0}n_{1}\cdots \allowbreak n_{L-1}\}\ni j$,  
$$
z^{[L]}_{k}(x,w) = \scalebox{1.0}{$\sum_{j}$}  \bar w_{k,j} \bar \sigma_{j} (x,w) \bar x_{j}, 
$$
where $\bar w_{k,j},\bar \sigma_{j} (x,w)$ and $\bar x_j$  represent $\overline W^{}_{kj_{L-1}j_{L-2}\dots j_{0}},\dot \sigma_{j_{L-1}j_{L-2}\dots j_{1}}(\allowbreak x,w)$ and $x_{j_{0}}$ respectively with the change of indices (i.e., $\sigma_{j} (x,w)$ and  $\bar x_{j}$ respectively contain the $n_0$ numbers and $n_{1}\cdots n_{L-1}$ numbers of the same copy of each $\dot \sigma_{j_{L-1}j_{L-2}\dots j_{1}}(x,w)$ and $x_{j_0}$). Note that $\sum_{j}$ represents  summation over all the paths from the input $x$ to the $k$-th output unit.

\paragraph{DAGs} Remember that every DAG has at least one topological ordering, which can be used to to create a layered structure    with possible skip  connections (e.g., see \citealt{healy2001layer,neyshabur2015norm}). 
In other words, we consider DAGs such that the pre-activation  vector of the $l$-th layer can be written as    
$$
z^{[l]}(x,w) = \sum_{l'=0}^{l-1}W^{(l,l')} \sigma^{[l']} \left(z^{[l']}(x,w)\right)
$$ 
with a boundary definition $\sigma^{(0)} \left(z^{[0]}(x,w)\right)\equiv x$, where $W^{(l,l')}\in \RR^{n_l \times n_{l'}}$ is a matrix of weight parameters connecting the $l'$-th layer to the $l$-th layer. Again, $W^{(l,l')}$ can have \textit{any} structure. Thus, in the same way as with layered networks without skip connections, for all $ k \in \{1,\dots,d_{y}\}$,  
$$
z^{[L]}_{k}(x,w) = \scalebox{1.0}{$\sum_{j}$}  \bar w_{k,j} \bar \sigma_{j} (x,w) \bar x_j,
$$ 
where  $\sum_{j}$ represents the summation over all paths from the input $x$ to the $k$-th output unit; i.e., $\bar w_{k,j} \bar \sigma_{j} (x,w) \bar x_j$ is the contribution from the $j$-th path to the $k$-th output unit. 
Each of $\bar w_{k,j}, \bar \sigma_{j} (x,w)$  and $\bar x_j$ is defined in the same manner as in the case of layered networks without skip connections. In other words, the $j$-th path weight   $\bar w_{k,j}$ is the product of  the weight parameters  in the $j$-th path, and 
  $\bar \sigma _{j}(x,w)$ is the product of the $0$-$1$ activations in the $j$-th path,  corresponding to ReLU nonlinearity and max pooling;  $\bar \sigma_j (x,w)=1$ if all units in the $j$-th path  are active, and $\bar \sigma_j (x,w)=0$ otherwise. Also, $\bar x_j$ is the input used in the $j$-th path. Therefore, for DAGs, including layered networks without skip connections, 
\begin{align} \label{eq:deep_to_linear}
z^{[L]}_{k}(x,w) = [\bar x \circ \bar \sigma (x,w)]^\top \bar w_k,
\end{align}
where $[\bar x \circ \bar \sigma (x,w)]_j=\bar x_j \bar \sigma_{j} (x,w)$ and $(\bar w_k)_j=\bar w_{k,j}$ are the vectors of the size of the number of the paths.

\subsection{Theoretical insights via tight theory for every  pair $(\PP,S)$} \label{sec:data_dependent_bounds}
Theorem \ref{thm:deterministic_bound} solves Open Problem 2 (and hence Open Problem 1) for neural networks with  squared loss by stating that the generalization gap of a $w$ with respect to a problem  $(\PP_{(X,Y)},S)$ is tightly analyzable  with theoretical insights,  based  only on the quality of the  $w$ and the pair $(\PP_{(X,Y)},S)$. We do \textit{not} assume that   $S$ is  generated randomly based on  some relationship with  $\PP_{(X,Y)}$; the theorem holds for any dataset, regardless of how it was generated.
Let $w^{\scalebox{.6}{$S$}}$ and $\bar w_k^{\scalebox{.6}{$S$}}$ be the parameter vectors $w$ and $\bar w_k$ learned with a dataset $S$. Let ${\mathcal R}[w^{\scalebox{.6}{$S$}}]$ and ${\mathcal R}_{S}[w^{\scalebox{.6}{$S$}}]$ be the expect risk and empirical risk of the model with the learned parameter $w^{\scalebox{.6}{$S$}}$. Let $z_i=[\bar x_i \circ \bar \sigma (x_i,w^{\scalebox{.6}{$S$}})]$. Let $G=\EE_{x,y\sim \PP_{(X,Y)}}[zz\T]-\frac{1}{m}\sum_{i=1}^m z_iz_i^\top$ and $v=\frac{1}{m}\sum_{i=1}^m y_{ik}z_i - \EE_{x,y \sim \PP_{(X,Y)}} [y_{k}z]$.  Given a matrix $M$, let $\lambda_{\max}(M)$ be the largest eigenvalue of $M$.

\begin{theorem} \label{thm:deterministic_bound}
Let $\{\lambda_j\}_j$ and $\{u_j\}_j$ be a set of eigenvalues and a corresponding orthonormal set of eigenvectors  of $G$. Let $\theta_{\bar w_{k},j}^{(1)}$ be the angle between $u_j$ and $\bar w_k$.  Let $\theta^{(2)}_{\bar w_{k}}$ be the angle between $v$ and $\bar w_k$. Then (deterministically),
\begin{align*}
{\mathcal R}[w^{\scalebox{.6}{$S$}}]- {\mathcal R}_{S}[w^{\scalebox{.6}{$S$}}] - c_{y} 
 &= \scalebox{.93}{$
 \sum\limits_{k=1}^{d_y} \left(2\|v\|_2\|\bar w_k^{\scalebox{.5}{$S$}}\|_2 \cos \theta^{(2)}_{\bar w_k^{\scalebox{.45}{$S$}}}+\|\bar w_k^{\scalebox{.5}{$S$}}\|_2^{2}\sum\limits_j \lambda_j  \cos^2 \theta_{\bar w_k^{\scalebox{.45}{$S$}},j}^{(1)} \right)
$}
\\ & \le\sum_{k=1}^{d_y} \left(2\|v\|_2\|\bar w_k^{\scalebox{.5}{$S$}}\|_2+\lambda_{\max}(G) \|\bar w_k^{\scalebox{.5}{$S$}}\|^2_2 \right), 
\end{align*} 
\normalsize
where $c_{y} = \EE_{y}[\|y\|_2^2]-\frac{1}{m}\sum_{i=1}^m \|y_{i}\|_2^2$.
\end{theorem}
\textit{Proof idea.} From Equation \eqref{eq:deep_to_linear} with  squared loss, we can decompose the generalization gap into three terms:
\begin{align} \label{eq:decomp_sq_loss}
{\mathcal R} [w^{\scalebox{.6}{$S$}}]- {\mathcal R}_{S}[w^{\scalebox{.6}{$S$}}]
&= \sum_{k=1}^{d_y}\sbr{(\bar w_k^{\scalebox{.5}{$S$}})\T \rbr{\EE[zz\T]-\frac{1}{m}\sum_{i=1}^m z_iz_i\T} \bar w_k^{\scalebox{.5}{$S$}}} 
\\ \nonumber & \ \ \ + 2\sum_{k=1}^{d_y}\sbr{\rbr{\frac{1}{m}\sum_{i=1}^m y_{ik}z_i\T - \EE [y_{k}z\T]}\bar w_k^{\scalebox{.5}{$S$}}}
 \\ \nonumber & \ \ \ + \EE[y\T y]-\frac{1}{m}\sum_{i=1}^m y_{i}\T y_i. 
\end{align}
By manipulating each term, we obtain the desired statement. 
 See Appendix \ref{app:proof_determ_bound} for a complete proof.  
$\hfill \square$

\vspace{10pt}

In Theorem \ref{thm:deterministic_bound}, there is no concept of a hypothesis space. Instead, it indicates  that if the norm of the weights $\|\bar w_k^{\scalebox{.5}{$S$}}\|_2$ at the end of learning process with the actual given $S$ is small, then the generalization gap is small, even if the norm $\|\bar w_k^{\scalebox{.5}{$S$}}\|_2$ is unboundedly large  at anytime with any  dataset other  than $S$.  

Importantly, in Theorem \ref{thm:deterministic_bound}, there are  two other significant factors in addition to the norm of the weights $\|\bar w_k^{\scalebox{.5}{$S$}}\|_2$. First, the eigenvalues of $G$ and $v$ measure the concentration of the given dataset $S$ with respect to the (unknown) $\PP_{(X,Y)}$ in the space of the learned representation $z_i=[\bar x_i \circ \bar \sigma (x_i,w^{\scalebox{.6}{$S$}})]$. Here, we can see the benefit of deep learning from the viewpoint of ``deep-path'' feature learning: even if a given $S$ is not concentrated in the original space, optimizing $w$ can result in concentrating it in the  space of $z$. Similarly,  $c_y$ measures the concentration of $\|y\|_2^2$, but $c_y$ is independent of $w$ and unchanged after a pair $(\PP_{(X,Y)},S)$ is given. Second, the $\cos \theta$ terms measure the similarity between  $\bar w_k^{\scalebox{.5}{$S$}}$  and these concentration terms. Because the norm of the weights $\|\bar w_k^{\scalebox{.5}{$S$}}\|_2$ is multiplied by those other factors, the generalization gap can remain small, even if $\|\bar w_k^{\scalebox{.5}{$S$}}\|_2$ is large, as long as  some of those other factors are small.

Based on  a generic bound-based theory, \citet{neyshabur2015path,neyshabur2015norm} proposed to control the norm of the \textit{path} weights  $\|\bar w_k\|_2$, which is consistent with our direct bound-less result (and which is  as computationally tractable as a  standard forward-backward pass\footnote{From the derivation of Equation \eqref{eq:deep_to_linear}, one can compute $\|\bar w_k^{\scalebox{.5}{$S$}}\|_2^2$ with a single forward pass using element-wise squared weights, an identity input, and no nonlinearity. One can also follow the  previous paper \citep{neyshabur2015path} for its computation.}).
Unlike the previous results, we do \textit{not} require a pre-defined bound on    $\|\bar w_k\|_2$ over different datasets, but depend only on its final value with each $S$ as desired, in addition to more tight insights (besides  the norm) via equality as discussed above. In addition to the pre-defined norm bound, these previous results have an explicit exponential dependence on the depth of the network, which does not appear in our Theorem 4. Similarly, some  previous results specific to layered networks without skip connections \citep{sun2016depth,xie2015generalization} contain the  $2^{L-1}$ factor \textit{and} a bound on the product of the norm of weight matrices,  $\prod_{l=1}^{L} \|W^{(l)}\|$, instead of $\sum_k\|\bar w_k^{\scalebox{.5}{$S$}}\|_2$. Here,  $\sum_k\|\bar w_k\|_2^{2} \le\prod_{l=1}^{L} \|W^{(l)}\|_F^{2}$ because the latter contains all of the same terms as the former as well as additional non-negative additive terms after expanding the sums in the definition of the norms.        

Therefore, unlike previous bounds, Theorem \ref{thm:deterministic_bound} generates these new theoretical insights based  on \textit{the tight equality} (in the first line of the equation in Theorem \ref{thm:deterministic_bound}). Notice that without manipulating the generalization gap, we can always obtain equality.  However, the question answered here is whether or not we can obtain competitive theoretical insights (the path norm bound) via equality instead of inequality. From a practical view point, if the obtained insights are the same (e.g., regularize the norm), then equality-based theory has the obvious advantage of being more precise.

\subsection{Probabilistic bound over random datasets} \label{sec:sub_data-independent}

While the previous subsection tightly analyzed each given point $(\PP_{(X,Y)},S)$, this subsection considers the set $ \mathcal P \times D \ni(\PP_{(X,Y)},S)$, where $D$ is the set of possible datasets $S$ endowed with an i.i.d. product measure $\PP_{(X,Y)}^{m}$ where\ $\PP_{(X,Y)} \in \mathcal P$ (see Section \ref{sec:reslove_ge_puzzle}).

In Equation \eqref{eq:decomp_sq_loss}, the generalization gap is decomposed into  three terms, each of which contains the difference between a sum of \textit{dependent} random variables and its expectation. 
The dependence comes  from the fact that     $z_{i}=[\bar x_i \circ \bar \sigma (x_i,w^{\scalebox{.6}{$S$}})]$  are dependent over the sample index $i$, because of the dependence of $w^{\scalebox{.6}{$S$}}$ on the entire dataset $S$. We then observe the following:  in $z^{[L]}_{k}(x,w) = [\bar x \circ \bar \sigma (x,w)]^\top \bar w$, 
the derivative  of $z=[\bar x \circ \bar \sigma (x,w)]$ with respect to $w$ is zero everywhere (except for the measure zero set, where the derivative does not exist). Therefore, each step of  the (stochastic) gradient decent greedily chooses the  best direction in terms of $\bar w$ (with the current $z=[\bar x \circ \bar \sigma (x,w)]$), but not in terms of $w$ in $z=[\bar x \circ \bar \sigma (x,w)]$  (see Appendix \ref{app:gradient_no_activation} for more detail). This observation leads to a conjecture that  the dependence of $z_i=[\bar x_i \circ \bar \sigma (x_i,w^{\scalebox{.6}{$S$}})]$ via the training process with the whole dataset   $S$ is not entirely ``bad''in terms of the concentration of  the sum of the terms with $z_{i}$. 

\subsubsection{Empirical observations} \label{sec:theory_sub_emp-obs} 

\begin{figure}[t!] 
\labellist 
\pinlabel \rotatebox{90}{Test accuracy ratio} [r] at 0 150
\pinlabel $\alpha$ [t] at 280 10
\endlabellist
\centering
\includegraphics[width=0.5\columnwidth]{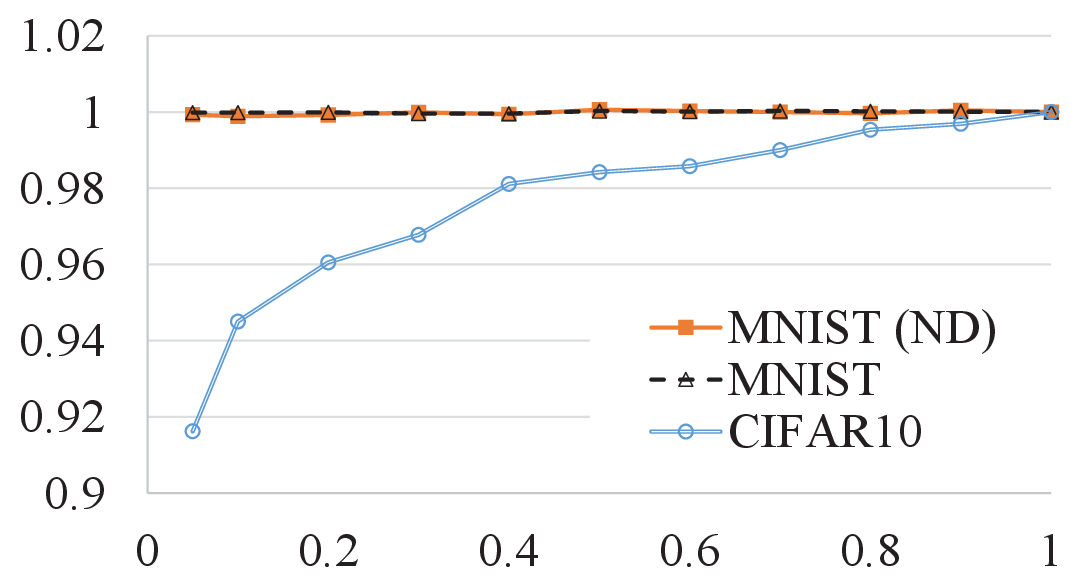} 
\caption{Test accuracy ratio (Two-phase/Base). Notice that the y-axis starts with  high initial accuracy, even with a very small dataset size  $\alpha m$ for learning $w_\sigma$.}
\label{fig:freeze} 
\end{figure}

As a first step to investigate the dependence of $z_i$, we evaluated the following novel \textit{two-phase} training procedure that explicitly breaks the dependence of $z_i$ over the sample index $i$. We first train a network in a standard way, but only using a \textit{partial} training dataset $S_{\alpha m}=\{(x_{1},y_{1}),\dots,(x_{\alpha m},y_{\alpha m})\}$ of size   $\alpha m$, where $\alpha \in (0,1)$ (standard phase). We then assign the value of $w^{\scalebox{.6}{$\mathcal S_{\alpha m}$}}$ to a new  placeholder $w_\sigma:=w^{\scalebox{.6}{$\mathcal S_{\alpha m}$}}$ and freeze $w_\sigma$, meaning that as $w$ changes, $w_\sigma$ does not change. At this point, we have that 
$
z^{[L]}_{k}(x,w^{\scalebox{.6}{$\mathcal S_{\alpha m}$}})=[\bar x \circ \bar \sigma (x,w_{\sigma})]^\top \bar w_k^{\scalebox{.6}{$\mathcal S_{\alpha m}$}}.
$
We then  keep training only the $\bar w_k^{\scalebox{.6}{$\mathcal S_{\alpha m}$}}$  part with the entire training dataset of size $m$ (freeze phase), yielding the final model via this two-phase training procedure as 
\begin{align} \label{eq:two-phase}
\tilde z^{[L]}_{k}(x,w^{\scalebox{.6}{$\mathcal S$}})=[\bar x \circ \bar \sigma (x,w_{\sigma})]^\top \bar w_k^{\scalebox{.6}{$\mathcal S$}}.
\end{align}
 Note that the vectors $w_\sigma=w^{\scalebox{.6}{$\mathcal S_{\alpha m}$}}$ and $\bar w_k^{\scalebox{.6}{$\mathcal S$}}$ contain  the untied parameters  in $\tilde z^{[L]}_{k}(x,w^{\scalebox{.6}{$\mathcal S$}})$. See Appendix \ref{app:implement_two-phase-train} for a simple   implementation of this two-phase training procedure that requires  at most (approximately) twice as much computational cost as the normal   training procedure. 

We implemented the two-phase training procedure with the MNIST and CIFAR-10 datasets. The test accuracies of the standard training procedure (base case) were 99.47\% for MNIST (ND), 99.72\% for MNIST, and 92.89\% for CIFAR-10. MNIST (ND) indicates MNIST with no data augmentation. The experimental details are in Appendix \ref{app:exp_setting1}. 

Our source code is available at: \url{http://lis.csail.mit.edu/code/gdl.html}  

Figure \ref{fig:freeze} presents the test accuracy ratios for varying $\alpha$: the test accuracy of the two-phase training procedure divided by the test accuracy of the standard training procedure.  The plot in Figure \ref{fig:freeze} begins with $\alpha=0.05$, for which $\alpha m=3000$ in MNIST\ and $\alpha m = 2500$ in CIFAR-10. Somewhat surprisingly, using a much smaller dataset for learning $w_\sigma$ still resulted in competitive performance. A dataset from which we could more easily obtain a better generalization (i.e., MNIST) allowed us to use smaller $\alpha m$ to achieve competitive performance, which is consistent with our discussion above.

\subsubsection{Theoretical results}

We now prove a probabilistic bound for the hypotheses resulting from the two-phase training algorithm. Let $\tilde z_i=[\bar x_i \circ \bar \sigma (x_i,w_{\sigma})]$ where $w_\sigma:=w^{\scalebox{.6}{$\mathcal S_{\alpha m}$}}$, as defined in the two-phase training procedure above. Our two-phase training procedure forces $\tilde z_{\alpha m+1},\dots,\tilde z_{m}$ \textit{over samples} to be independent random variables (each $\tilde z_i$ is dependent \textit{over} \textit{coordinates}, which is taken care of in our proof), while maintaining the competitive practical performance of the output model $\tilde z^{[L]}_{k}(\cdot \hspace{0.5pt},w^{\scalebox{.6}{$\mathcal S$}})$. 
As a result, we obtain the following  bound on the generalization gap for the practical deep models $\tilde z^{[L]}_{k}(\cdot \hspace{0.5pt},w^{\scalebox{.6}{$\mathcal S$}})$. Let $m_{\sigma}=(1-\alpha)m$.
 Given a matrix $M$, let $\|M\|_2$ be the spectral norm of  $M$.

\begin{assumption} \label{assumpt:direct1}
 Let $G^{(i)}=\EE_x[\tilde z\tilde z\T]-\tilde z_{i}\tilde z_{i}\T$, $V^{(i)}_{kk'}=y_{ik}\tilde z_{i,k'}-\EE_{x,y} [y_{k} \tilde z_{k'}]$, and $c_{y}^{(i)}=\EE_y[\|y\|^2_2] - \|y_{i}\|_2^2$. Assume that for all $i \in \{\alpha m+1,\dots,m\}$,
\begin{itemize}
\item 
$C_{zz}\ge \lambda_{\max}(G^{(i)}) $  and $\gamma^2_{zz} \ge \| \EE _{x}[(G^{(i)})^2] \|_2$
\item  
\small 
$C_{yz} \ge \max_{k,k'} |V^{(i)}_{kk'}|$ and $\gamma^2_{yz} \ge \max_{k,k'} \EE _{x}[(V^{(i)}_{kk'})^2] )$
\normalsize
\item  
$ C_{y} \ge |c_y^{(i)}| $  and $\gamma^2_{y} \ge \EE _{x}[(c_y^{(i)})^2] $.
\end{itemize}
\end{assumption}

\begin{theorem} \label{thm:theoretical_concern_regression}
Suppose that Assumption \ref{assumpt:direct1} holds. Assume that  $S \setminus S_{\alpha m}$ is  generated by i.i.d. draws according to true distribution $\PP_{(X,Y)}$. Assume that $S \setminus S_{\alpha m}$ is independent of $S_{\alpha m}$. Let $f_{\mathcal A(S)}$ be the model learned by the two-phase training procedure with $S$. Then, for each   $w_\sigma:=w^{\scalebox{.6}{$\mathcal S_{\alpha m}$}}$, for any $\delta>0$, with probability at least $1-\delta$,
\begin{align*}
{\mathcal R}[f_{\mathcal A(S)}] - {\mathcal R}_{S \setminus S_{\alpha m}}[f_{\mathcal A(S)}]  
\le  \beta_1\sum_{k=1}^{d_y}\left\|\bar w_k^{\scalebox{.65}{$S$}}\right\|_1 + 2 \beta_2 \small \sum_{k=1}^{d_y}\left\|\bar w_k^{\scalebox{.65}{$S$}}\right\|_2^2+\beta_3, 
\end{align*} 
where
$
\beta_1=\frac{2C_{zz}}{3m_{\sigma }} \ln \frac{3d_{z}}{\delta} + \sqrt{\frac{2\gamma^2_{zz}}{m_{\sigma}} \ln \frac{3d_{z}}{\delta}},
$ 
$
\beta_2=\frac{2C_{yz}}{3m_{\sigma}} \ln \frac{6d_{y}d_z}{\delta} + \sqrt{\frac{\gamma^2_{yz}}{m_{\sigma}} \ln \frac{6d_{y}d_z}{\delta}},
$ 
and 
$
\beta_2=\frac{2C_{y}}{3m_{\sigma}} \ln \frac{3}{\delta} + \sqrt{\frac{2\gamma^2_{y}}{m_{\sigma}} \ln \frac{3}{\delta}}.
$
\end{theorem}       

\vspace{5pt}

Our proof does \textit{not} require  independence over the coordinates of  $\tilde z_i$ and the entries of the random matrices $\tilde z_i \tilde z_i\T$(see the proof of Theorem \ref{thm:theoretical_concern_regression}). 

The bound in Theorem \ref{thm:theoretical_concern_regression}  is data-dependent because the norms of the weights $\bar w_k^{\scalebox{.65}{$S$}}$ depend on each particular $S$. Similarly to Theorem \ref{thm:deterministic_bound}, the bound in Theorem \ref{thm:theoretical_concern_regression}  does not contain a pre-determined bound on the norms of weights and can be independent of the  concept of  hypothesis space, as desired; i.e., Assumption \ref{assumpt:direct1} can be also satisfied without referencing a hypothesis space of $w$, because $\tilde z=[\bar x_i \circ \bar \sigma (x_i,w_{\sigma})]$ with $\bar \sigma_j (x_i,w_{\sigma}) \in \{0,1\}$. However, unlike  Theorem \ref{thm:deterministic_bound}, Theorem \ref{thm:theoretical_concern_regression}   \textit{implicitly} contains the properties of datasets different from a given $S$, via the pre-defined bounds in  Assumption \ref{assumpt:direct1}. This is expected since Theorem \ref{thm:theoretical_concern_regression}  makes claims about the set of random datasets $S$ instead of each instantiated $S$. Therefore, while Theorem \ref{thm:theoretical_concern_regression}   presents a strongly-data-dependent bound (over random datasets),  Theorem \ref{thm:deterministic_bound} is tighter for each given $S$; indeed, the main equality of Theorem \ref{thm:deterministic_bound}  is as tight as possible.

Theorems \ref{thm:deterministic_bound} and \ref{thm:theoretical_concern_regression} provide generalization  bounds for practical deep learning models that  do not necessarily have explicit dependence on the number of weights, or exponential dependence on depth  or effective input dimensionality. Although the size of the vector  $\bar w_k^{\scalebox{.65}{$S$}}$  can be exponentially large in the depth of the network, the norms of the vector need not be. Because $\tilde z^{[L]}_{k}(x,w^{\scalebox{.65}{$S$}})= \|\bar x \circ \bar \sigma (x,w_\sigma)\|_{2} \|\bar w_k^{\scalebox{.65}{$S$}} \|_2 \cos \theta$, we have that $\|\bar w_k^{\scalebox{.65}{$S$}} \|_2 = z^{[L]}_{k}(x,w) / (\|\bar x \circ \bar \sigma (x,w_\sigma)\|_{2}\cos \theta)$ (unless the denominator is zero), where $\theta$ is the angle between $\bar x \circ \bar \sigma (x,w_\sigma)$ and $\bar w_k^{\scalebox{.65}{$S$}}$. Additionally, as discussed in Section \ref{sec:data_dependent_bounds}, $\sum_k\|\bar w_k\|_2^{2} \le  \prod_{l=1}^{L} \|W^{(l)}\|_F^{2}$.

\subsection{Probabilistic bound for 0-1 loss with multi-labels} \label{sec:classification_direct}
For the 0--1 loss with multi-labels, the two-phase training procedure in Section \ref{sec:sub_data-independent} yields the generalization bound in Theorem \ref{thm:theoretical_concern_classification}. Similarly to the bounds in Theorems \ref{thm:deterministic_bound} and \ref{thm:theoretical_concern_regression}, the generalization  bound  in  Theorem \ref{thm:theoretical_concern_classification}   does not necessarily have dependence on the number of weights, and exponential dependence on depth and effective input dimensionality.

\begin{theorem} \label{thm:theoretical_concern_classification}
Assume that  $S \setminus S_{\alpha m}$ is  generated by i.i.d. draws according to true distribution $\PP_{(X,Y)}$. Assume that $S \setminus S_{\alpha m}$ is independent of $S_{\alpha m}$. Fix $\rho>0$ and $w_\sigma$. Let ${\mathcal F}$ be the set of the models with the two-phase training procedure. Suppose that $\EE_x[ \|\bar x \circ \bar \sigma (x,w_\sigma)\|^2_2] \le C_{\sigma}^2$ and $\max_k \|\bar w_k\|_2 \le C_{w}$ for all  $f \in {\mathcal F}$. Then, for any $\delta>0$, with probability at least $1-\delta$,
the following holds for all $f \in {\mathcal F}$:
$$
{\mathcal R}[f] \le {\mathcal R}_{S \setminus S_{\alpha m}}^{(\rho)}[f] +\frac{2 d_{y}^2 (1-\alpha)^{-1/2}C_{\sigma}C_w}{\rho \sqrt{m_\sigma}}  + \sqrt \frac{\ln \frac{1}{\delta}}{2m_{\sigma}}.
$$
\end{theorem} 
Here, the  empirical margin  loss ${\mathcal R}_{S}^{(\rho)}[f]$ is defined as ${\mathcal R}_{S}^{(\rho)}[f]=\frac{1}{m}\sum_{i=1}^m  \allowbreak {\mathcal L}_{\margin,\rho}(f(x_i),y_i)$, where  ${\mathcal L}_{\margin,\rho}$ is defined as follows: 
 $$
 {\mathcal L}_{\margin,\rho}(f(x),y)={\mathcal L}_{\margin,\rho}^{(2)} ( {\mathcal L}_{\margin,\rho}^{(1)}(f(x),y))
 $$  
 where 
$$
{\mathcal L}_{\margin,\rho}^{(1)}(f(x),y)=z^{[L]}_{y} (x)- \max_{y \neq y'} z^{[L]}_{y'} (x)\in \RR,
$$
and 
$$
{\mathcal L}_{\margin,\rho}^{(2)}(t)=
\begin{cases} 0 & \text{if } \rho \le t \\
1-t/\rho &  \text{if } 0 \le t \le \rho \\
1 & \text{if } t \le 0.\\
\end{cases}
$$

\section{Discussions and  open problems} \label{sec:discussion_openp}

It is very difficult to make a detailed characterization of how well a specific hypotheses generated by a certain learning
algorithm will generalize, in the absence of detailed information about the given problem instance.
Traditional learning theory addresses this very difficult question and has developed bounds that are as tight
as possible given the generic information available. In this paper, we have worked toward drawing stronger conclusions
by developing theoretical analyses tailored for the situations with more detailed information, including  actual neural network structures, and  actual performance on a validation set.

Optimization and generalization in deep learning are closely related via the following observation: if we make optimization easier by changing model architectures,  generalization performance can be degraded, and vice versa. Hence, non-pessimistic  generalization theory discussed in this paper might allow more architectural choices  and assumptions in optimization theory.

Define the partial order of problem instances $(\PP_{},S,f)$ as
$$
(\PP,S,f) \le(\PP',S',f')  \ \ \ \Leftrightarrow \ \ \ {\mathcal R}_\PP[f] - {\mathcal R}_{S}[f] \le {\mathcal R}_{\PP'}[f'] - {\mathcal R}_{S'}[f'] 
$$
where ${\mathcal R}_\PP[f]$ is the expected risk with probability measure $\PP$. Then, any theoretical insights without the partial order preservation can  be  misleading as it can change the ranking of the preference of $(\PP,S,f)$. For example, theoretically motivated algorithms  can degrade actual performances when compared with heuristics, if the theory does not preserve the partial order of $(\PP,S,f)$. This observation suggests the following open problem.

\vspace{6pt}  
\noindent \textbf{Open Problem 3.} 
Tightly characterize the expected risk ${\mathcal R}[f]$ or the generalization gap  ${\mathcal R}[f] - {\mathcal R}_{S}[f]$ of a hypothesis $f$ with a pair $(\PP,S)$, producing theoretical insights while partially yet provably preserving the   partial order of  $(\PP,S,f)$. 
\vspace{6pt}

Theorem \ref{thm:deterministic_bound} partially addresses Open Problem 3 by preserving the exact ordering via equality without bounds. However,  it would  be beneficial to  consider a weaker notion of  order preservation to gain analyzability  with more useful insights as stated in  Open Problem 3.

Our discussion  with Proposition \ref{prop:ge_validation} and Remark \ref{rem:ge_validation_2}  suggests another open problem:  analyzing the role and influence of \textit{human intelligence} on generalization.  For example,  human intelligence seems to be able to often find good architectures (and other hyper-parameters)
that get low validation errors (without  non-exponentially large  $|{\mathcal F}_\val|$ in Proposition \ref{prop:ge_validation}, or a low complexity of $\mathcal L_{{\mathcal F}_{\val}}$ in Remark \ref{rem:ge_validation_2}). 
A close look at the deep learning literature seems to suggest that this question is fundamentally related to the process of science and engineering, because many  successful  architectures have been designed based on the physical properties and engineering priors of the problems at hand (e.g., hierarchical nature, convolution, architecture for motion such as that by \citealt{finn2016unsupervised}, memory networks, and so on).
While this is a hard question, understanding it would be beneficial to  further automate the role of human intelligence towards the goal of artificial intelligence.

\section*{Acknowledgements} 
We gratefully acknowledge support from NSF grants 1420316, 1523767 and 1723381,  from AFOSR FA9550-17-1-0165, from ONR grant N00014-14-1-0486, and from ARO grant W911 NF1410433, as well as support from NSERC, CIFAR and Canada Research Chairs. Any opinions, findings, and conclusions or recommendations expressed in this material are those of the authors and do not necessarily reflect the views of our sponsors.

\bibliography{all}
\bibliographystyle{apalike}

\renewcommand{\thesection}{A}
\setcounter{section}{0}
\renewcommand{\theequation}{A.\arabic{equation}}
\setcounter{equation}{0}
\renewcommand{\thefigure}{A.\arabic{figure}}
\setcounter{figure}{0}
\renewcommand{\thesubsection}{A\arabic{subsection}}

\section{Appendix: Additional discussions} \label{app:additional_discussions}

Theorem \ref{thm:deterministic_bound} address  Open Problem 2  with the limited applicability to  certain neural networks with squared loss. In contrast,  a  parallel study \citep{kawaguchi2018alt} presents a novel generic learning theory to address Open Problem 2 for general cases in machine learning. It would be  beneficial to explore both  a generic analysis \citep{kawaguchi2018alt} and a concrete analysis in deep learning to get theoretical insights that are tailored for  each particular  case.

In previous bounds with a hypothesis space ${\mathcal F}$, if we try different hypothesis spaces ${\mathcal F}$ depending on $S$, the basic proof breaks down. An easy recovery at the cost of an extra quantity in a bound is to take a union bound over all possible ${\mathcal F}_j$ for $j=1,2,\dots,$ where we pre-decide  $\{{\mathcal F}_j\}_j$ without dependence on $S$ (because simply considering the ``largest'' ${\mathcal F}\supseteq {\mathcal F}_j$ can result in a very loose bound for each ${\mathcal F}_j$). Similarly, if we need to try many $w_\sigma^{}:=w^{\scalebox{.6}{$\mathcal S_{\alpha m}$}}$ depending on the whole $S$ in Theorem  \ref{thm:theoretical_concern_regression},  we can  take a union bound over $w_\sigma^{(j)}$ for $j=1,2,\dots,$  where we pre-determine $\{w_\sigma^{(j)}\}_j$  without dependence on $S \setminus S_{\alpha m}$ but with dependence on $S_{\alpha m}$. We can do the same with Proposition \ref{prop:ge_validation} and Remark \ref{rem:ge_validation_2} to use many ${\mathcal F}_{\val}$ depending on the validation dataset $S_{m_\val}^{(\val)}$ with a predefined sequence. 
\subsection{Simple regularization algorithm}  \label{app:method}
In general, theoretical bounds  from statistical learning theory can be too loose to be directly used in practice.  In addition, many theoretical results in statistical learning theory end up suggesting to simply regularize some notion of smoothness of a hypothesis class. Indeed, by upper bounding a distance between two functions (e.g., a hypothesis and the ground truth function corresponding to expected true labels), one can immediately see  \textit{without  statistical learning theory} that regularizing some notion of smoothness of the hypothesis class  helps guarantees on generalization. Then, by the Occam's razor argument, one might prefer a simpler (yet still rigor) theory and a corresponding simpler algorithm. 

Accordingly, this subsection examines another simple regularization algorithm that directly regularize smoothness of the learned hypothesis. In this subsection, we focus on multi-class classification problems with  $d_{y}$ classes, such as object classification with images. Accordingly, we analyze the expected risk with $0$--$1$ loss as  
 $
 {\mathcal R}[f] = \EE_x[\mathbbm{1}\{f(x)=y(x)\}],
 $ 
where $f(x)=\argmax_{k \in \{1,\dots,d_{y}\}}(\allowbreak z^{[L]}_k(x))$ is the model prediction, and $y(x)\in \{1,\dots,d_{y}\}$ is the true label  of $x$.

This subsection proposes the following family of simple regularization algorithms: given any architecture and method, add a new regularization term for each mini-batch as  
$$
\text{\small loss = original loss} + \frac{\lambda}{\bar m}  \abr{\max_{\substack{k}}  \sum_{i=1}^{\bar m} \xi_i z_k^{[L]}(x_i)},  
$$
\normalsize
where $x_i$ is drawn from some distribution approximating the true distribution of $x$, $\xi_{1},\dots,\xi_{\bar m}$ are independently and uniformly drawn from $\{-1,1\}$, $\bar m$ is a mini-batch size and $\lambda$ is a hyper-parameter. Importantly, the approximation of the true distribution of $x$ is only used for regularization purposes  and hence needs not be  precisely accurate (as long as it plays its role for regularization). For example, it can be approximated by populations generated by a generative neural network and/or an extra data augmentation process. For simplicity, we call this  family of methods as Directly Approximately Regularizing Complexity (DARC).

In this paper,  we  evaluated only a very simple version of the proposed family of methods as a first step. That is, our experiments employed the following simple and easy-to-implement method, called DARC1:
\begin{align} \label{eq:reg_simplest}
\text{\small loss = original loss} + \frac{\lambda}{\bar m} \left(\max_{k } \sum_{i=1}^{\bar m} |z_k^{[L]}(x_i)|\right),  
\end{align} 
\normalsize
where $x_i$ is the $i$-th sample in the training mini-batch. The additional computational cost and programming effort due to this new regularization is  almost negligible because    $z_k^{[L]}(x_i)$ is already used in computing the original loss. This simplest version was derived by approximating the true distribution of $x$ with the empirical distribution of the training data.

\begin{table}[t!]
\centering
\caption{Test error (\%)   } \label{tbl:best_models} 
\begin{tabular}{lcc}
\toprule
Method & MNIST & CIFAR-10 \\
\midrule
Baseline & 0.26 & 3.52 \\
\midrule
DARC1 & \uline{0.20} & \uline{3.43} \\
\bottomrule
\end{tabular} 
\end{table} 

\begin{table}[t!]
\centering
\caption{Test error ratio (DARC1/Base)} \label{tbl:short-runs_mean_stdv}
\begin{tabular}{l|cc|cc|cc}
\toprule
\multirow{2}{*}{\small } &  \multicolumn{2}{|c|}{\small MNIST (ND)}  & \multicolumn{2}{|c|}{ MNIST} & \multicolumn{2}{|c}{CIFAR-10} \\
  & \small mean & \small stdv & \small mean &\small stdv &\small mean &\small stdv \\
\midrule
Ratio & 0.89 & 0.61  & 0.95 & 0.67 & 0.97 & 0.79     \\
\bottomrule
\end{tabular}
\end{table} 

\begin{table}[t!] 
\centering
\caption{Values of $\frac{1}{m} \left(\max_{k } \sum_{i=1}^{m} |z_k^{[L]}(x_i)|\right)$} \label{tbl:short-runs_reg-val} 
\begin{tabular}{l|cc|cc|cc}
\toprule
\multirow{2}{*}{\small Method} &  \multicolumn{2}{|c|}{\small MNIST (ND)}  & \multicolumn{2}{|c|}{ MNIST} & \multicolumn{2}{|c}{CIFAR-10} \\
  & \small mean & \small stdv & \small mean &\small stdv &\small mean &\small stdv \\
\midrule
Base &  17.2 & 2.40 & 8.85 & 0.60 & 12.2 & 0.32 \\
\midrule
\small DARC1  &  1.30 & 0.07 & 1.35 & 0.02 & 0.96 & 0.01 \\
\bottomrule
\end{tabular}
\end{table}

We evaluated the proposed method (DARC1) by simply adding the new regularization term in equation \eqref{eq:reg_simplest}  to existing standard codes for MNIST and CIFAR-10. A standard variant of LeNet \citep{lecun1998gradient} and ResNeXt-29($16\times64$d) \citep{xie2016aggregated} are used for MNIST and CIFAR-10, and compared with the addition
  of the studied regularizer. For all the experiments, we fixed $(\lambda/\bar m)=0.001$  with $\bar m =64$. We  used a single model without ensemble methods. The experimental details are in Appendix \ref{app:exp_setting1}. The source code is available at: \url{http://lis.csail.mit.edu/code/gdl.html}

Table \ref{tbl:best_models} shows the error rates comparable with previous results. To the best of our knowledge, the previous state-of-the-art classification error is 0.23\% for  MNIST with a single model \citep{sato2015apac} (and 0.21\%  with an ensemble by \citealt{wan2013regularization}). 
To further investigate the improvement, we ran 10 random trials with  computationally less expensive settings, to gather mean and standard deviation (stdv). For MNIST, we used fewer epochs with the same model. For CIFAR-10, we used a smaller model class (pre-activation ResNet with only 18 layers). Table \ref{tbl:short-runs_mean_stdv} summarizes the improvement ratio: the new model's error divided by the base model's error. 
We observed the improvements   for all cases. The test errors (standard deviations) of the base models were 0.53 (0.029) for MNIST (ND), 0.28 (0.024) for MNIST, and 7.11 (0.17) for CIFAR-10 (all in $\%$). 

Table \ref{tbl:short-runs_reg-val} summarizes the values of  the regularization term $\frac{1}{m} (\max_{k } \sum_{i=1}^{m} \allowbreak |z_k^{[L]}(x_i)|)$ for each obtained  model. The models learned with the proposed method were significantly different from the base models  in terms of this value. Interestingly, a comparison of the base cases for MNIST (ND) and MNIST shows that  data augmentation by itself  \textit{implicitly} regularized what we explicitly regularized in the proposed method.

\subsection{Relationship to other fields} \label{app:relation_other-fields}
The situation  where theoretical studies  focus on a set of problems and practical applications care  about each element in the set is prevalent in machine learning and computer science literature, not limited to the field of learning theory. For example,  for each practical problem instance $q\in Q$, the size of the set $Q$ that had been analyzed in theory for optimal exploration in Markov decision processes (MDPs) were  demonstrated to be frequently too pessimistic, and a methodology to partially mitigate the issue was  proposed \citep{kawaguchiAAAI2016}. Bayesian optimization  would suffer from  a pessimistic set $Q$ regarding each  problem instance $q\in Q$, the issue of which was partially mitigated  \citep{kawaguchiNIPS2015}. 

Moreover, characterizing a set of problems $Q$ only via a  worst-case instance $q'\in Q$ (i.e., worst-case analysis) is known to have several issues in theoretical computer science, and so-called \textit{beyond worst-case analysis} (e.g., smoothed analysis) is an active area of research to mitigate the issues.

\subsection{SGD chooses direction in terms of $\bar w$} \label{app:gradient_no_activation}
Recall that
$$
z^{[L]}_{k}(x,w) = z\T \bar w = [\bar x \circ \bar \sigma (x,w)]^\top \bar w_.
$$
Note that $\sigma (x,w)$ is 0 or 1 for max-pooling and/or ReLU nonlinearity. Thus, the derivative  of $z=[\bar x \circ \bar \sigma (x,w)]$ with respect to $w$ is zero everywhere (except at the measure zero set where the derivative does not exists). Thus, by the chain rule (and power rule), the gradient of the loss with respect to $w$ only contain the contribution from the derivative of $z^{[L]}_{k}$ with respect to $\bar w$, but not with respect to $w$ in $z$.     

\subsection{Simple implementation of two-phase training procedure} \label{app:implement_two-phase-train}
Directly implementing Equation \eqref{eq:two-phase} requires the summation over all paths, which can be computationally expensive. To avoid it, we implemented it by creating two deep neural networks, one of which defines $\bar w$ paths hierarchically, and another of  which defines $w_\sigma$ paths hierarchically, resulting in the computational cost at most (approximately) twice as much as the original cost of training standard deep learning models. We tied $w_\sigma$ and $\bar w$ in the two networks during standard phase, and untied them during freeze phase. 

Our source code is available at: \url{http://lis.csail.mit.edu/code/gdl.html}

The  computation  of the standard network without skip connection can be re-written as:   
\begin{align*}
z^{[l]}(x,w) &= \sigma^{[l]}(W^{[l]}z^{[l-1]}(x,w)) 
\\ &= \dot{\sigma}^{[l]}(W^{[l]}z^{[l-1]}(x,w)) \circ W^{[l]}z^{[l-1]}(x,w)
\\ &= \dot{\sigma}^{[l]}(W^{[l]}_\sigma z^{[l-1]}_{\sigma}(x,w)) \circ W^{[l]}z^{[l-1]}(x,w)
\end{align*}
where $W^{[l]}_\sigma:=W^{[l]}$, $z^{[l-1]}_{\sigma}:=\sigma(W^{[l]}_{\sigma}z^{[l-1]}_\sigma(x,w))$ and $\dot{\sigma}^{[l]}_{j}(W^{[l]}z^{[l-1]}(x,w)) \allowbreak = 1$ if the $j$-th unit at the $l$-th layer is active, $\dot{\sigma}^{[l]}_{j}(W^{[l]}z^{[l-1]}(x,w)) = 0$ otherwise.
Note that because $W^{[l]}_\sigma=W^{[l]}$, we have that $z^{[l-1]}_{\sigma}=z^{[l]}$ in standard phase and standard models. 

 In the two-phase training procedure, we created two networks for $W^{[l]}_\sigma \allowbreak z^{[l-1]}_{\sigma}(x,w)$ and $ W^{[l]} \allowbreak z^{[l-1]}(x,w)$ separately. We then set $W^{[l]}_\sigma=W^{[l]}$ during standard phase, and frozen $W^{[l]}_\sigma$ and only trained $W^{[l]}$  during freeze phase. By following  the same derivation of Equation \eqref{eq:deep_to_linear}, we can see that this defines the desired computation without explicitly computing the summation over all paths. By the same token, this applies to DAGs.

\renewcommand{\thesection}{B}
\setcounter{section}{0}
\renewcommand{\theequation}{B.\arabic{equation}}
\setcounter{equation}{0}
\renewcommand{\thefigure}{B.\arabic{figure}}
\setcounter{figure}{0}
\renewcommand{\thesubsection}{B\arabic{subsection}}

\section{Appendix: Experimental details} \label{app:exp_setting1}

\noindent\uline{For MNIST:} 

We used the following fixed architecture:  

\begin{enumerate}[label=(\roman*)]
\item 
Convolutional layer  with 32 filters with filter size of 5 by 5, followed
by max pooling of size of 2 by 2 and ReLU.
\item
Convolution layer with 32 filters with filter size of 5 by 5, followed
by max pooling of size of 2 by 2 and ReLU.
\item
Fully connected layer with output 1024 units, followed by ReLU and Dropout with its probability being 0.5.
\item
Fully connected layer with output 10 units.
\end{enumerate}

Layer 4 outputs $z^{[L]}$ in our notation. For training purpose, we use  softmax of $z^{[L]}$. Also, $f(x)=\argmax(z^{[L]}(x))$ is the label prediction.

We fixed learning rate to be 0.01, momentum coefficient to be 0.5, and optimization algorithm to be (standard) stochastic gradient decent (SGD). We fixed  data augmentation process as:   random crop with size 24,  random rotation up to $\pm 15$ degree, and scaling of 15\%. We used 3000 epochs for Table \ref{tbl:best_models}, and 1000 epochs for Tables \ref{tbl:short-runs_mean_stdv} and \ref{tbl:short-runs_reg-val}.

\vspace{10pt}\noindent \uline{For  CIFAR-10:} 

For data augmentation, we used random horizontal flip with probability 0.5 and random crop of size 32 with padding of size 4. 

For Table \ref{tbl:best_models}, we used ResNeXt-29($16\times64$d) \citep{xie2016aggregated}. We set initial learning rate to be 0.05 and decreased to $0.005$ at 150 epochs, and to 0.0005 at 250 epochs. We fixed momentum coefficient to be 0.9, weight decay coefficient to be $5\times 10^{-4}$, and optimization algorithm to be stochastic gradient decent (SGD) with Nesterov momentum. We stopped training at 300 epochs.  

For Tables \ref{tbl:short-runs_mean_stdv} and \ref{tbl:short-runs_reg-val}, we used pre-activation ResNet with only 18 layers (pre-activation ResNet-18) \citep{he2016identity}. We fixed learning rate to be 0.001 and momentum coefficient to be 0.9, and optimization algorithm to be (standard) stochastic gradient decent (SGD). We used 1000 epochs.

\renewcommand{\thesection}{C}
\setcounter{section}{0}
\renewcommand{\theequation}{C.\arabic{equation}}
\setcounter{equation}{0}
\renewcommand{\thefigure}{C.\arabic{figure}}
\setcounter{figure}{0}
\renewcommand{\thesubsection}{C\arabic{subsection}}

\section{Appendix: Proofs} \label{app:sec3}
We use the following lemma in the proof of Theorem \ref{thm:deterministic_bound}.  

\begin{lemma} \label{lem:matrix_Bernstein}
\emph{(Matrix  Bernstein inequality: corollary to theorem 1.4 in \citealt{tropp2012user})}
Consider a finite sequence $\{M_i\}$ of independent, random, self-adjoint matrices with dimension $d$. Assume that each random matrix satisfies that
$\EE[M_i] =0$ and $\lambda_{\max}(M_i) \le R$ almost surely. Let $\gamma^2=\|\sum_i \EE[M_i^2]\|_2$. Then, for any $\delta>0$, with probability at least $1-\delta$, 
$$
\lambda_{\max}\rbr{\sum_i M_i} \le   \frac{2R}{3} \ln \frac{d}{\delta} + \sqrt{2\gamma^2 \ln \frac{d}{\delta}}.
$$
\end{lemma}
\begin{proof}
Theorem 1.4 by \citet{tropp2012user} states that for all $t \ge 0,$
$$
\PP\sbr{\lambda_{\max} \rbr{\sum_i M_i} \ge t} \le d \cdot \exp \rbr{\frac{-t^2 /2}{\gamma^2 + Rt/3}}.
$$
Setting 
$
\delta= d\exp\rbr{-\frac{t^2/2}{\gamma^2+Rt/3}}
$
implies 
$$
-t^2+\frac{2}{3}R (\ln d/\delta) t +2 \gamma^2 \ln d/\delta =0. 
$$
Solving for $t$ with the quadratic formula and bounding the solution with the subadditivity   of square root on non-negative terms (i.e., $\sqrt{a+b} \le \sqrt a + \sqrt b$ for all $a,b\ge0$),
$$
t \le\frac{2}{3}R (\ln d/\delta)+2 \gamma^2 \ln d/\delta. 
$$
\end{proof}

\subsection{Proof of Theorem \ref{thm:linear-couter}}  \label{app:theorem_linear}

\begin{proof}
For any matrix $M$, let $\Col(M)$ and $\Null(M)$ be the column space and null space of $M$.
Since $\rank(\Phi) \ge m$ and $\Phi \in \RR^{m \times n}$, the set of its columns span $\RR^{m}$, which proves statement \ref{thm:linear-counter-i}. Let $w^* = w_1^* + w_2^*$ where  $\Col(w^{*}_1) \subseteq \Col(M^T)$ and $\Col(w_2^*) \subseteq \Null(M)$.  For statement \ref{thm:linear-counter-ii}, set the  parameter as $w := w^*_{1}+\epsilon C_1 + \alpha C_2$ where $\Col(C_1) \subseteq \Col(M^T)$,  $\Col(C_2) \subseteq \Null(M)$, $\alpha \ge 0$ and $C_2=\frac{1}{\alpha}w^*_2+\bar C_2$.  Since $\rank(M) < n$, $\Null(M) \neq \{0\}$ and  there exist non-zero $\bar C_2$. Then, 
$$
\hat Y(w) = Y + \epsilon\Phi C_{1}, 
$$ 
and 
$$
\hat Y_\test(w) = Y_\test + \epsilon\Phi_{\test}C_{1}. 
$$
By setting $A=\Phi C_{1}$ and $B=\Phi_{\test}C_{1}$ with a proper normalization of $C_1$ yields \ref{thm:linear-counter-ii-a} and \ref{thm:linear-counter-ii-b} in statement \ref{thm:linear-counter-ii} (note that $C_1$ has an arbitrary freedom in the bound on its scale because its only condition is $\Col(C_1) \subseteq \Col(M^T)$). At the same time with the same parameter, since $\Col(w^*_{1}+\epsilon C_1) \perp \Col(C_2)$, 
$$
\|w\|_{F}^2 =\|w^*_{1}+\epsilon C_1\|_F^2 +\alpha^2 \|C_2\|_{F}^2,  
$$
and
$$ 
\|w-w^*\|_F^2 =\|\epsilon C_1\|_F^2 +\alpha^2 \|\bar C_2\|_{F}^2,
$$    
which grows unboundedly as $\alpha \rightarrow \infty$ without changing $A$ and $B$, proving \ref{thm:linear-counter-ii-c} in statement \ref{thm:linear-counter-ii}.
\end{proof}

\subsection{Proof of Corollary \ref{coro:linear-counter}} \label{app:coro_linear}
\begin{proof}
It follows the fact that the proof in Theorem \ref{thm:linear-couter} uses the assumption of $n > m$ and $\rank(\Phi) \ge m$ only for statement \ref{thm:linear-counter-i}.
\end{proof}

\subsection{Proof of Theorem \ref{thm:deterministic_bound}} \label{app:proof_determ_bound}
\begin{proof}
From Equation \eqref{eq:deep_to_linear}, the squared loss of deep models for each  point $(x,y)$ can be rewritten as 
$$
\sum_{k=1}^{d_y}(z^\top \bar w_{k}-y_k)^2=\sum_{k=1}^{d_y}\bar w\T_{k}(z z^\top) \bar w_{k}-2y_{k}z^\top \bar w_{k} +y_k^{2}.
$$ 

Thus, from Equation \eqref{eq:deep_to_linear} with the squared loss, we can decompose the generalization gap into three terms  as
\begin{align*} 
{\mathcal R} [w^{\scalebox{.6}{$S$}}]- {\mathcal R}_{S}[w^{\scalebox{.6}{$S$}}]
&= \sum_{k=1}^{d_y}\sbr{(\bar w_k^{\scalebox{.5}{$S$}})\T \rbr{\EE[zz\T]-\frac{1}{m}\sum_{i=1}^m z_iz_i\T} \bar w_k^{\scalebox{.5}{$S$}}} 
\\ & \ \ \ + 2\sum_{k=1}^{d_y}\sbr{\rbr{\frac{1}{m}\sum_{i=1}^m y_{ik}z_i\T - \EE [y_{k}z\T]}\bar w_k^{\scalebox{.5}{$S$}}}
 \\ & \ \ \ + \EE[y\T y]-\frac{1}{m}\sum_{i=1}^m y_{i}\T y_i. 
\end{align*}
\vskip-0.4cmAs $G$ is a real symmetric matrix, we denote an eigendecomposition of $G$ as $G=U \Lambda U\T$ where the diagonal matrix $\Lambda$ contains  eigenvalues as $\Lambda_{jj}=\lambda_j$ with  the corresponding orthogonal eigenvector matrix $U$;  $u_j$ is the $j$-th column of $U$. Then,
\begin{align*}
(\bar w_k^{\scalebox{.5}{$S$}})\T G \bar w_k^{\scalebox{.5}{$S$}} =\sum_j \lambda_j (u_{j}\T \bar w_k^{\scalebox{.5}{$S$}})^{2} = \|\bar w_k^{\scalebox{.5}{$S$}}\|_2^{2}\sum_j \lambda_j   \cos^2 \theta_{\bar w_k^{\scalebox{.45}{$S$}},j}^{(1)},
\end{align*}
and 
\small
$$
\sum_j \lambda_j (u_{j}\T \bar w_k^{\scalebox{.5}{$S$}})^{2}\le \lambda_{\max}(G) \|U\T \bar w_k^{\scalebox{.5}{$S$}}\|^2_2=\lambda_{\max}(G) \|\bar w_k^{\scalebox{.5}{$S$}}\|^2_2.  
$$
\normalsize
Also, 
$$
v\T \bar w_k^{\scalebox{.5}{$S$}}=\|v\|_2\|\bar w_k^{\scalebox{.5}{$S$}}\|_2 \cos \theta^{(2)}_{\bar w_k^{\scalebox{.5}{$S$}}} \le\|v\|_2\|\bar w_k^{\scalebox{.5}{$S$}}\|_2.
$$ 

By using these, 
\begin{align*} 
&\ {\mathcal R} [w^{\scalebox{.6}{$S$}}]- {\mathcal R}_{S}[w^{\scalebox{.6}{$S$}}] - c_{y}  
\\ &= \scalebox{1.}{$
 \sum_{k=1}^{d_y} \left(2\|v\|_2\|\bar w_k^{\scalebox{.5}{$S$}}\|_2 \cos \theta^{(2)}_{\bar w_k^{\scalebox{.45}{$S$}}}+\|\bar w_k^{\scalebox{.5}{$S$}}\|_2^{2}\sum_j \lambda_j  \cos^2 \theta_{\bar w_k^{\scalebox{.45}{$S$}},j}^{(1)} \right)
$}
\\ & \le\sum_{k=1}^{d_y} \left(2\|v\|_2\|\bar w_k^{\scalebox{.5}{$S$}}\|_2+\lambda_{\max}(G) \|\bar w_k^{\scalebox{.5}{$S$}}\|^2_2 \right). 
\end{align*} 
\normalsize
\end{proof}

\subsection{Proof of Theorem \ref{thm:theoretical_concern_regression}}
\begin{proof}
We do \textit{not} require the independence over the coordinates of  $\tilde z_i$ and the entries of random matrices $\tilde z_i \tilde z_i\T$because of the definition of independence required  for matrix  Bernstein inequality (for $\frac{1}{m_{\sigma}}\sum_{i=1}^{m_{\sigma}} \tilde z_i \tilde z_i\T$) (e.g., see section 2.2.3 of \citealt{tropp2015introduction}) and because of a union bound over coordinates (for $\frac{1}{m_\sigma}\sum_{i=1}^{m_\sigma} y_{ik}\tilde z_{i}  $). We use the fact that  $\tilde z_{\alpha m+1},\dots,\tilde z_{m}$  are independent random variables \textit{over the sample index (although dependent over the coordinates)}, because a $w_\sigma:=w^{\scalebox{.6}{$\mathcal S_{\alpha m}$}}$ is fixed and independent from $x_{\alpha m+1},\dots,x_{m}$.  

From Equation \eqref{eq:decomp_sq_loss}, with the definition of induced matrix norm and the Cauchy-Schwarz inequality,
%\small  
\begin{align} \label{eq:proof_probabilistic_bound}
&{\mathcal R} [f_{\mathcal A(S)}]- {\mathcal R}_{S \setminus S_{\alpha m}}[f_{\mathcal A(S)}] 
\\ \nonumber & \le \sum_{k=1}^{d_y} \left\|\bar w_k^{\scalebox{.65}{$S$}}\right\|_2^2 \lambda_{\max}\rbr{\EE[\tilde z\tilde z\T]-\frac{1}{m_\sigma}\sum_{i=\alpha m+1}^m \tilde z_i\tilde z_i\T }
\\ \nonumber & \ \ \ + 2\sum_{k=1}^{d_y}\|\bar w_k^{\scalebox{.65}{$S$}}\|_1 \left\|\frac{1}{m_{\sigma}}\sum_{i=\alpha m+1}^m y_{ik}\tilde z_i - \EE [y_{k}\tilde z] \right\|_\infty 
\\ \nonumber & \ \ \ + \EE[y\T y]-\frac{1}{m_\sigma}\sum_{i=\alpha m+1}^m y_{i}\T y_i.
\end{align}  
 
In the below, we bound each term of the right-hand side of the above with concentration inequalities. 

\uline{For the first term:} Matrix  Bernstein inequality (Lemma \ref{lem:matrix_Bernstein}) states that for any $\delta>0$, with probability at least $1-\delta/3$,
\begin{align*}
 \lambda_{\max}\rbr{{\EE[\tilde z\tilde z\T]-\frac{1}{m_\sigma}\sum_{i=\alpha m+1}^{m} \tilde z_{i}\tilde z_{i}\T }} 
\le    \frac{2C_{zz}}{3m_\sigma} \ln \frac{3d_{z}}{\delta} + \sqrt{\frac{2\gamma^2_{zz}}{m_\sigma} \ln \frac{3d_{z}}{\delta}}.
\end{align*}

Here, Matrix  Bernstein inequality was applied as follows. Let $M_i =(\frac{1}{m_\sigma} \allowbreak G^{(i)})$. Then, $\sum_{i=\alpha m+1}^{m} \allowbreak  M_i =\EE[\tilde z\tilde z\T]-\frac{1}{m_\sigma}\sum_{i=\alpha m+1}^{m} \tilde z_{i}\tilde z_{i}\T$. We have that $\EE[ M_i]=0$ for all $i$. Also, $\lambda_{\max}(M_i) \le \frac{1}{m_\sigma} C_{zz}$ and $\|\sum_i \EE [M^2_i]\|_2 \le\frac{1}{m_\sigma}\gamma^2_{zz}$.

\uline{For the second term:} We apply Bernstein inequality to each $(k,k') \in \{1,\dots,d_y\}\times \{1,\dots, d_z\}$ and take union bound over $d_yd_z$ events, obtaining that for any $\delta>0$, with probability at least $1-\delta/3$, for all $k \in \{1,2,\dots,d_y\}$,
\begin{align*}
 \nm{\frac{1}{m_\sigma}\sum_{i=\alpha m+1}^{m} y_{ik}\tilde z_{i} - \EE [y_{k}\tilde z] }_\infty 
 \le\frac{2C_{yz}}{3m_\sigma} \ln \frac{6d_{y}d_z}{\delta} + \sqrt{\frac{\gamma^2_{yz}}{m_\sigma} \ln \frac{6d_{y}d_z}{\delta}}  
\end{align*}

\uline{For the third term:}
From Bernstein inequality, with probability at least $1-\delta/3$,
$$
 \EE[y\T y]-\frac{1}{m_\sigma}\sum_{i=\alpha m+1}^{m} y_{i}\T y_i \le  \frac{2C_{y}}{3m} \ln \frac{3}{\delta} + \sqrt{\frac{2\gamma^2_{y}}{m} \ln \frac{3}{\delta}}.
$$

\uline{Putting together:}  Putting together, for a fixed (or frozen) $w_\sigma$, with probability at least $1-\delta$ (probability over $S \setminus S_{\alpha m}=\{(x_{\alpha m+1},y_{\alpha m+1}),\dots,(x_{m}, \allowbreak y_m)\}$),  we have that $$\lambda_{\max}\rbr{{\EE[\tilde z\tilde z\T]-\frac{1}{m_\sigma}\sum_{i=\alpha m+1}^{m} \tilde z_{i}\tilde z_{i}\T }} \le \beta_1$$, $\nm{\frac{1}{m_\sigma}\sum_{i=\alpha m+1}^{m} y_{ik}\tilde z_{i} - \EE [y_{k}\tilde z] }_\infty \le \beta_2$ (for all $k$), and $$\EE[y\T y]-\frac{1}{m_\sigma}\sum_{i=\alpha m+1}^{m} y_{i}\T y_i \le \beta_3.$$Since Equation \eqref{eq:proof_probabilistic_bound} always hold deterministically (with or without such a dataset), the desired statement of this theorem follows. 
\end{proof}

\subsection{Proof of Theorem \ref{thm:theoretical_concern_classification}}

\begin{proof}

Define $S_{m_\sigma}$ as
$$
S_{m_\sigma}=S \setminus S_{\alpha m}=\{(x_{\alpha m+1},y_{\alpha m+1}),\dots,(x_{m},y_m)\}.
$$
Recall the following fact: using the result by \citet{koltchinskii2002empirical}, we have that for any $\delta>0$, with probability at least $1-\delta$,
the following holds for all $f \in {\mathcal F}$:
$$
{\mathcal R}[f] \le {\mathcal R}_{S_{m_\sigma},\rho}[f] +\frac{2d_{y}^2}{\rho m_\sigma}  \mathfrak{R}'_{m_\sigma} ({\mathcal F}) + \sqrt \frac{\ln \frac{1}{\delta}}{2m_\sigma},
$$ 
where $\mathfrak{R}'_{m_\sigma} ({\mathcal F})$ is Rademacher complexity defined as
$$
\mathfrak{R}'_{m_\sigma} ({\mathcal F}) = \EE_{S_{m_\sigma},\xi} \sbr{\sup_{\substack{k,w}}  \sum_{i=1}^{m_\sigma} \xi_i z_k^{[L]}(x_i,w)}.
$$
Here, $\xi_{i}$ is the Rademacher variable, and the supremum  is taken over all $k \in \{1,\dots,d_y\}$ and all $w$  allowed in ${\mathcal F}$.
Then, for our parameterized hypothesis spaces with any frozen $w_\sigma$,
\begin{align*}
 \mathfrak{R}'_{m_\sigma} ({\mathcal F}) 
&= \EE_{S_{m_\sigma},\xi} \sbr{\sup_{\substack{k,\bar w_k}}  \sum_{i=1}^{m_\sigma} \xi_i[\bar x_i \circ \bar \sigma (x_i,w_\sigma)]^\top \bar w_{k}}
\\ & \le \EE_{S_{m_\sigma},\xi} \sbr{\sup_{\substack{k,\bar w_k^{}}}  \left\| \sum_{i=1}^{m_\sigma} \xi_i[\bar x_i \circ \bar \sigma (x_i,w_\sigma)]\right\|_2 \left\|\bar w_{k}\right\|_2} 
\\ & \le C_{w}\EE_{S_{m_\sigma},\xi} \sbr{ \left\| \sum_{i=1}^{m_\sigma} \xi_i[\bar x_i \circ \bar \sigma (x_i,w_\sigma)]\right\|_2 }. 
\end{align*}
Because square root is concave in its domain, by using Jensen's inequality and linearity of expectation, 
%\fontsize{8pt}{8pt}
\begin{align*}
& \EE_{S_{m_\sigma},\xi} \sbr{ \left\| \sum_{i=1}^{m_\sigma} \xi_i[\bar x_i \circ \bar \sigma (x_i,w_\sigma)]\right\|_2 }
\\ & \le \left( \EE_{S_{m_\sigma}} \sum_{i=1}^{m_\sigma} \sum_{j=1}^{m_\sigma} \EE_{\xi}[\xi_i \xi_j] [\bar x_i \circ \bar \sigma (x_i,w_\sigma)]^\top [\bar x_j \circ \bar \sigma (x_j,w_\sigma)]  \right)^{1/2} 
\\ & = \left( \sum_{i=1}^{m_\sigma} \EE_{S_{m_\sigma}}\sbr{\left\|[\bar x_i \circ \bar \sigma (x_i,w_\sigma)] \right\|_2^2}  \right)^{1/2}
\\ & \le  C_{\sigma}  \sqrt {m_\sigma}.
\end{align*} 
\normalsize
Putting together, we have that $\mathfrak{R}'_m ({\mathcal F}) \le C_{\sigma}C_{w} \sqrt {m_\sigma}$.
\end{proof}

\subsection{Proof of Proposition \ref{prop:ge_validation}}
\begin{proof}
Consider a single fixed $f\in {\mathcal F}_\val$. Since ${\mathcal F}_\val$ is independent from the validation dataset, $\kappa_{f,1},\dots,\kappa_{f,m_\val}$ are independent zero-mean random variables, given a fixed $f\in {\mathcal F}_\val$. Thus, we can apply Bernstein inequality, yielding
$$   
\PP\rbr{\frac{1}{m_\val} \sum_{i=1}^{m_\val} \kappa_{f,i} > \epsilon} \le \exp\left(-\frac{\epsilon^2 m_\val /2}{\gamma^2+\epsilon C/3}\right).
$$
By taking union bound over all elements in ${\mathcal F}_\val$,
\begin{align*}
\PP\rbr{ \cup_{f\in {\mathcal F}_\val}\cbr{\frac{1}{m_\val} \sum_{i=1}^{m_\val} \kappa_{f,i} > \epsilon }} 
 \le |{\mathcal F}_{\val}|\exp\left(-\frac{\epsilon^2 m_\val /2}{\gamma^2+\epsilon C/3}\right).
\end{align*}
By setting $\delta= |{\mathcal F}_{\val}|\exp\left(-\frac{\epsilon^2 m_\val /2}{\gamma^2+\epsilon C/3}\right)$ and solving for $\epsilon$ (via quadratic formula), 
\small
$$
\epsilon = \frac{2C \ln (\frac{|{\mathcal F}_\val|}{\delta})}{6m_\val} \pm \frac{1}{2} \sqrt{\rbr{\frac{2C \ln (\frac{|{\mathcal F}_\val|}{\delta})}{3m_\val}}^2 + \frac{8\gamma^2 \ln (\frac{|{\mathcal F}_\val|}{\delta})}{{m_\val}}}.
$$ 
\normalsize
By noticing that the solution of $\epsilon$ with the minus sign results in $\epsilon<0$, which is invalid for Bernstein inequality, we obtain the valid solution with the plus sign. Then, we have 
$$
\epsilon \le\frac{2C \ln (\frac{|{\mathcal F}_\val|}{\delta})}{3m_\val} +  \sqrt{\frac{2\gamma^2 \ln (\frac{|{\mathcal F}_\val|}{\delta})}{{m_\val}}}, 
$$
where we used that $\sqrt{a+b} \le \sqrt a + \sqrt b$. By tanking the negation of the  statement, we obtain that for any $\delta>0$, with probability at least $1-\delta$, for all $f \in {\mathcal F}_\val$,
\begin{align*}
\frac{1}{m_\val} \sum_{i=1}^{m_\val} \kappa_{f,i} \le \frac{2C \ln (\frac{|{\mathcal F}_\val|}{\delta})}{3m_\val} +  \sqrt{\frac{2\gamma^2 \ln (\frac{|{\mathcal F}_\val|}{\delta})}{m}},
\end{align*}
where $\frac{1}{m_\val} \sum_{i=1}^{m_\val} \kappa_{f,i}={\mathcal R}[f] - {\mathcal R}_{S_{m_\val}^{(\val)}}[f]. $
\end{proof}

\renewcommand{\thesection}{\thechapter.\arabic{section}}
\setcounter{section}{2}
\renewcommand{\theequation}{\thechapter.\arabic{equation}}
\setcounter{equation}{0}
\renewcommand{\thefigure}{\thechapter.\arabic{figure}}
\setcounter{figure}{0}
\renewcommand{\thesubsection}{\thesection.\arabic{subsection}}

\end{document}